\documentclass[twoside,11pt]{article}
\usepackage[abbrvbib, nohyperref]{jmlr2e}

\usepackage{bm}
\usepackage{mathrsfs}
\usepackage{mathtools}
\usepackage{dsfont}
\usepackage{graphicx}
\usepackage{tikz}
\usepackage{hyperref}

\usetikzlibrary{arrows.meta}

\usepackage[capitalise,nameinlink]{cleveref}
\hypersetup{bookmarksopen=true}

\usepackage{crossreftools}
\crefformat{equation}{\textup{#2(#1)#3}}
\crefrangeformat{equation}{\textup{#3(#1)#4--#5(#2)#6}}
\crefmultiformat{equation}{\textup{#2(#1)#3}}{ and \textup{#2(#1)#3}}
{, \textup{#2(#1)#3}}{, and \textup{#2(#1)#3}}
\crefrangemultiformat{equation}{\textup{#3(#1)#4--#5(#2)#6}}{ and \textup{#3(#1)#4--#5(#2)#6}}{, \textup{#3(#1)#4--#5(#2)#6}}{, and \textup{#3(#1)#4--#5(#2)#6}}

\Crefformat{equation}{#2Equation~\textup{(#1)}#3}
\Crefrangeformat{equation}{Equations~\textup{#3(#1)#4--#5(#2)#6}}
\Crefmultiformat{equation}{Equations~\textup{#2(#1)#3}}{ and \textup{#2(#1)#3}}
{, \textup{#2(#1)#3}}{, and \textup{#2(#1)#3}}
\Crefrangemultiformat{equation}{Equations~\textup{#3(#1)#4--#5(#2)#6}}{ and \textup{#3(#1)#4--#5(#2)#6}}{, \textup{#3(#1)#4--#5(#2)#6}}{, and \textup{#3(#1)#4--#5(#2)#6}}

\crefname{Remark}{Remark}{Remarks}
\crefname{Lemma}{Lemma}{Lemmas}
\crefname{Definition}{Definition}{Definitions}
\crefname{Corollary}{Corollary}{Corollaries}
\crefname{Proposition}{Proposition}{Propositions}
\crefname{figure}{Figure}{Figures}

\usepackage{jmlr2e}

\usepackage[inline]{enumitem}
\usepackage[utf8]{inputenc}
\usepackage{xcolor}

\usepackage{macros}

\usepackage{lastpage}
\jmlrheading{22}{2021}{1-\pageref{LastPage}}{6/20; Revised
11/20}{2/21}{20-583}{Rahul Parhi and Robert D. Nowak}
\ShortHeadings{Representer Theorems for Neural Networks and Ridge Splines}{Parhi and Nowak}
\firstpageno{1}

\begin{document}

\title{Banach Space Representer Theorems for Neural Networks \\ and Ridge Splines}

\author{\name Rahul Parhi \email rahul@ece.wisc.edu  \\
        \name Robert D.\ Nowak \email rdnowak@wisc.edu \\
        \addr Department of Electrical and Computer Engineering\\
        University of Wisconsin--Madison\\
        Madison, WI 53706, USA}

\editor{Lorenzo Rosasco}

\date{\today}

\maketitle

\begin{abstract}We develop a variational framework to understand the properties of the functions learned by neural networks fit to data. We propose and study a family of continuous-domain linear inverse problems with total variation-like regularization in the Radon domain subject to data fitting constraints. We derive a representer theorem showing that finite-width, single-hidden layer neural networks are solutions to these inverse problems. We draw on many techniques from variational spline theory and so we propose the notion of polynomial ridge splines, which correspond to single-hidden layer neural networks with truncated power functions as the activation function. The representer theorem is reminiscent of the classical reproducing kernel Hilbert space representer theorem, but we show that the neural network problem is posed over a non-Hilbertian Banach space. While the learning problems are posed in the continuous-domain, similar to kernel methods, the problems can be recast as finite-dimensional neural network training problems. These neural network training problems have regularizers which are related to the well-known weight decay and path-norm regularizers.  Thus, our result gives insight into functional characteristics of trained neural networks and also into the design neural network regularizers. We also show that these regularizers promote neural network solutions with desirable generalization properties.
 \end{abstract}


\section{Introduction} \label{sec:intro}

Single-hidden layer neural networks are superpositions of \emph{ridge functions}. A ridge function is any multivariate function mapping $\R^d \to \R$ of the form
\begin{equation}
    \vec{x} \mapsto \rho(\vec{w}^\T \vec{x}),
    \label{eq:ridge-function}
\end{equation}
where $\rho: \R \to \R$ is a univariate real-valued function and  $\vec{w} \in \R^d\setminus \curly{\vec{0}}$. Single-hidden layer neural networks, in particular, are superpositions of the form
\begin{equation}
    \vec{x} \mapsto \sum_{k=1}^K v_k \, \rho(\vec{w}_k^\T \vec{x} - b_k),
    \label{eq:nn-as-superposition}
\end{equation}
where $\rho: \R \to \R$ is the \emph{activation function},
$K$ is the \emph{width} of the network, and for $k = 1, \ldots, K$, $v_k \in \R$ and $\vec{w}_k \in \R^d \setminus \curly{\vec{0}}$ are the \emph{weights} of the neural network and $b_k \in \R$ are the \emph{biases} or \emph{offsets}. 

This paper focuses on the \emph{practical problem} of fitting a finite-width neural network to finite-dimensional data, with an eye towards characterizing the properties of the resulting functions. We view this problem as a function recovery problem, where we wish to recover an \emph{unknown function} from \emph{linear measurements}. We deviate from the usual finite-dimensional recovery paradigm and pose the problem in the continuous-domain, allowing us to use techniques from the theory of variational methods.  We show that continuous-domain linear inverse problems with total variation regularization in the Radon domain admit sparse atomic solutions, with the atoms being the familiar neurons of a neural network.

\subsection{Contributions}
Let $\native$ be a topological vector space of multivariate functions, $\sensing: \native \to \R^N$ a continuous linear \emph{sensing} or \emph{measurement} operator ($N$ can be viewed as the number of measurements or data)\footnote{For example, $\sensing f = (f(\vec{x}_1), \ldots, f(\vec{x}_N)) \in \R^N$, for some data $\curly{\vec{x}_n}_{n=1}^N \subset \R^d$.}, and let $f: \R^d \to \R$ be a multivariate function such that $f \in \native$. Consider the continuous-domain inverse problem
\begin{equation}
    \min_{f \in \native} \: \datafit(\sensing f) + \norm{f},
    \label{eq:generic-inverse-problem}
\end{equation}
where $\norm{\dummy}: \native \to \R_{\geq 0}$ is a (semi)norm or \emph{regularizer} and $\datafit: \R^N \to \R$ is a convex \emph{data fitting} term.

We summarize the contributions of this paper below.
\begin{enumerate}
    \item Our main result is the development of a family of seminorms $\norm{\dummy}_{(m)}$, where $m \geq 2$ is an integer, so that the solutions to the problem \cref{eq:generic-inverse-problem} with $\norm{\dummy} \coloneqq \norm{\dummy}_{(m)}$ take the form
    \begin{equation}
        \vec{x} \mapsto \sum_{k=1}^K v_k \, \rho_m(\vec{w}_k^\T \vec{x} - b_k) + c(\vec{x}),
        \label{eq:generic-solution-inverse-problem}
    \end{equation}
    where $\rho_m = \max\curly{0, \dummy}^{m-1} / (m - 1)!$ is the $m$th-order \emph{truncated power function}, $c(\dummy)$ is a polynomial of degree strictly less than $m$, and $K \leq N$. These seminorms are inspired by the seminorm proposed in~\cite{function-space-relu}, which is equivalent to $\norm{\dummy}_{(m)}$ with $m=2$. Specifically, the seminorm $\norm{f}_{(m)}$ is the total variation (TV) norm (in the sense of measures) of $\partial_t^m\ramp^{d-1} \RadonOp f$, where $\RadonOp$ is the Radon transform, $\ramp^{d-1}$ is a ``ramp'' filter, and $\partial_t^m$ is the $m$th partial derivative with respect to the ``offset'' variable of the Radon domain. In other words, our main result is the derivation of a \emph{neural network representer theorem}.  Our result says that single-hidden layer neural networks are solutions to continuous-domain linear inverse problems with TV regularization in the Radon domain. When $m = 2$, the solutions correspond to ReLU networks.
    
    \item We propose the notion of a \emph{ridge spline} by noticing that our problem formulation in \cref{eq:generic-inverse-problem} is similar to those  studied in variational spline theory~\citep{splines-variational, splines-sobolev-seminorm, L-splines}, with the key twist being that our family of seminorms are in the Radon domain. Thus, we refer to the solutions \cref{eq:generic-solution-inverse-problem} with our family of seminorms as \emph{$m$th-order polynomial ridge splines} to emphasize that the solutions are superpositions of ridge functions. We view our notion of a ridge spline as a kind of spline in-between a univariate spline and a traditional multivariate spline. Unlike polyharmonic splines, the usual multivariate analogue of univariate polynomial splines, ridge splines are multivariate piecewise polynomial functions. Moreover, by specializing our result to the univariate case, our notion of a ridge spline exactly coincides with the notion of a univariate polynomial spline.
    
    \item By specializing our main result to setting in which $\sensing$ corresponds to \emph{ideal sampling}, i.e., point evaluations, the generality of \cref{eq:generic-inverse-problem} allows us to consider the machine learning problem of approximating the scattered data $\curly{(\vec{x}_n, y_n)}_{n=1}^N \subset \R^d \times \R$ with both \emph{interpolation constraints} in the case of noise-free data as well as \emph{regularized problems} where we have soft-constraints in the case of noisy data. Thus, a direct consequence of our representer theorem result says the infinite-dimensional problem in \cref{eq:generic-inverse-problem} can be recast as a \emph{finite-dimensional neural network training problem} with various regularizers that are related to weight decay~\citep{weight-decay} and path-norm~\citep{path-norm} regularizers, which are used in practice. In other words, a neural network trained to fit data with an appropriate regularizer is ``optimal'' in the sense of the seminorm $\norm{\dummy}_{(m)}$, characterizing a key property of the learned function. We also note that in these neural network training problems, it is sufficient that the width $K$ of the network be $N$, the size of the data.
    
    \item Specializing our results to the  supervised learning problem of binary classification shows that neural network solutions with small seminorm make good predictions on new data. Binary classification corresponds to the ideal sampling setting, restricting ourselves to $y_n \in \curly{-1, +1}$, $n = 1, \ldots, N$, and predicting these by the sign of the function that solves \cref{eq:generic-inverse-problem} (this can be done with an appropriate data fitting term). We derive statistical \emph{generalization bounds} for the class of neural networks with uniformly bounded seminorm $\norm{\dummy}_{(m)}$. In particular, we show that the seminorm bounds the \emph{Rademacher complexity} of these neural networks and use standard results from machine learning theory to relate this to the generalization error. This says that a small seminorm implies good generalization properties.
\end{enumerate}

\subsection{Related work}
Ridge functions are ubiquitous in mathematics and engineering, especially due to the popularity of neural networks, and we refer to the book of~\cite{ridge-functions-book} and the survey of~\cite{ridge-functions-survey} for a fairly up-to-date treatment on the current state of research regarding ridge functions. One of the most popular areas of research has been regarding approximation theory with superpositions of ridge functions (i.e., single-hidden layer neural networks).  Variants of the well-known universal approximation theorem state that \emph{any} continuous function can be approximated arbitrarily well by a superposition of the form in \cref{eq:nn-as-superposition}, under extremely mild conditions on the activation function~\citep{uat1, uat2, uat3, uat4, nn-approx-non-poly}. There are also many papers establishing optimal or near-optimal approximation rates for various function spaces~\citep{approx-Lp,approx-ridge-splines, dimension-independent-approx-bounds}.

Another, less popular (though practically more interesting), research area studies what happens when you fit data with a single-hidden layer neural network. This question has been viewed from both a statistical perspective, where risk bounds are established~\citep{risk-bounds-ridge}, and more recently, in the univariate case, from a functional analytic perspective, where connections to variational spline theory are established~\citep{relu-linear-spline,gradient-dynamics-shallow,min-norm-nn-splines}. We also remark that these questions have also been studied in the context of deep neural networks. See, for example,~\cite{deep-approx1,deep-approx2} for approximation theory,~\cite{statistical-deep} for statistical properties, and~\cite{balestriero2018spline,representer-deep, convex-duality-deep} for connections to splines.

Although the term \emph{ridge function} is rather modern, it is important to note that such functions have been studied for many years under the name \emph{plane waves}. Much of the early work with plane waves revolves around representing solutions to partial differential equations (PDE), e.g., the wave equation, as a superposition of plane waves. We refer the reader to the classic book of~\cite{plane-waves-pdes} for a full treatment of this subject. The key analysis tool used in these PDE problems is the \emph{Radon transform}. Since a ridge function as in \cref{eq:ridge-function} is constant along the hyperplanes $\vec{w}^\T \vec{x} = c$, $c \in \R$, analysis of such functions becomes convenient in the \emph{Radon domain}. More modern applications of ridge functions arise in computerized tomography following the seminal paper of~\cite{computerized-tomography}, where they coined the term ``ridge function'', and the development of ridgelets in the 1990s, a wavelet-like system inspired by neural networks, independently proposed by~\cite{murata-ridgelets},~\cite{rubin-ridgelets}, and~\cite{candes-phd, candes-ridgelets}. Many refinements to the ridgelet transform have been made recently~\citep{ridgelet-transform-distributions,ridgelet-uat}. As one might expect, the main analysis tool used in these applications is the Radon transform. Thus, we see that ridge functions and the Radon transform are intrinsically connected.

Recent work from the machine learning community has used this connection to understand what kinds of functions can be represented by \emph{infinite-width} (continuum-width) single-hidden layer neural networks with Rectified Linear Unit (ReLU) activation functions, where the ``norm'' of the network weights is bounded~\citep{function-space-relu}. They ask the question about what functions can be represented by such infinite-width, but bounded norm, networks. They show that a TV seminorm in the Radon domain exactly captures the Euclidean norm of the network weights, but do not address the optimization problem of fitting neural networks to data. Inspired by this seminorm, we develop and study a family of TV seminorms in the Radon domain and consider the problem of scattered data approximation. We show that single-hidden layer neural networks, with fewer neurons than data, are solutions to the problem of minimizing these seminorms over the space of all functions in which the seminorms are well-defined, subject to data fitting constraints. A side effect of our analysis also provides an understanding of the topological structure, specifically a non-Hilbertian Banach space structure, of the spaces defined by these seminorms.

Although our main result might seem obvious on a surface level, actually proving it is quite delicate. The problem of learning from a continuous dictionary of atoms with TV-like regularization has been studied before, both in the context of splines~\citep{fisher-jerome,locally-adaptive-regression-splines} and machine learning~\citep{l1-prob,convex-nn}. It is extremely important to note that all of these prior works make the assumption that the relevant spaces are compact. This allows appealing to standard arguments which are useful for proving, e.g., that minimizers to their problem even exist. We also remark that some of these prior works simply assume, without proof, existence of minimizers.

Since the Radon domain is an unbounded domain, we cannot appeal to these types of arguments for the problem we study. Thus, a very important question we ask, and subsequently answer, regards existence of solutions to \cref{eq:generic-inverse-problem} with our family of seminorms. To this end, we draw on techniques from the recently developed variational framework of $\Ell$-splines~\citep{L-splines}. We also remark that we cannot directly apply the results from this framework since the fundamental assumption about splines is that spline atoms are translates of a single function. Meanwhile, neural network atoms as in \cref{eq:nn-as-superposition} are parameterized by both a direction $\vec{w}_k$ and a translation $b_k$. We also draw on recent results from variational methods~\citep{sparsity-variational-inverse}. Thus, the results of this paper provide a general variational framework as well as novel insights into understanding the properties of functions learned by neural networks fit to data.

\subsection{Roadmap} In \cref{sec:main-results} we state our main results and highlight some of the technical challenges and novelties in proving our results. In \cref{sec:prelim} we introduce the notation and mathematical formulation used throughout the paper. In \cref{sec:rep-thm} we prove our main result, the representer theorem. In \cref{sec:splines} we discuss connections between ridge splines and classical polynomial splines. In \cref{sec:nn-training} we discuss applications of the representer theorem to neural network training, regularization, and generalization.
     \section{Main Results} \label{sec:main-results}
Our main contribution is a representer theorem for problems of the form  in \cref{eq:generic-inverse-problem} with our proposed family of seminorms. Our other contributions are (rather straightforward) corollaries to this result. In this section we will state the main results of this paper along with relevant historical remarks.

\subsection{The representer theorem} \label{subsec:rep-thm}
The notion of a \emph{representer theorem} is a fundamental result regarding kernel methods~\citep{spline-rep-thm, generalized-rep-thm, learning-with-kernels}. In particular, let $(\mathcal{H}, \norm{\dummy}_{\mathcal{H}})$ be any real-valued Hilbert space on $\R^d$ and consider the scattered data $\curly{(\vec{x}_n, y_n)}_{n=1}^N \subset \R^d \times \R$. The classical representer theorem considers the variational problem
\begin{equation}
    \bar{f} = \argmin_{f \in \mathcal{H}} \sum_{n=1}^N \ell(f(\vec{x}_n), y_n) + \lambda \norm{f}_\mathcal{H}^2,
    \label{eq:kernel-opt}
\end{equation}
where $\ell(\cdot, \cdot)$ is a convex loss function and $\lambda > 0$ is an adjustable regularization parameter. The representer theorem then states that the solution $\bar{f}$ is unique and $\bar{f} \in \spn\curly{k(\dummy, \vec{x}_n)}_{n=1}^N$, where $k(\dummy, \dummy)$ is the \emph{reproducing kernel} of $\mathcal{H}$. Kernel methods (even before the term ``kernel methods'' was coined) have received much success dating all the way back to the 1960s, especially due to the tight connections between kernels, reproducing kernel Hilbert spaces, and splines~\citep{splines-minimum, scattered-data, spline-models-observational}.

Recently, the term ``representer theorem'' has started being used for general problems of convex regularization~\citep{L-splines, rep-thm-convex-reg,unifying-representer} as a way to designate a parametric formulation of solutions to a variational optimization problem, ideally being a linear combination from some dictionary of atoms. This has allowed more much more general problems to be considered than ones like \cref{eq:kernel-opt}, which are restricted to regularizers which are  Hilbertian (semi)norms. In particular, some of the recent theory is able to to consider problems where the search space is a locally convex topological vector space and the regularizers being a seminorm defined on that space~\citep{sparsity-variational-inverse}. The main utility of these more general representer theorems arise in understanding \emph{sparsity-promoting} regularizers such as the $\ell^1$-norm or its continuous-domain analogue, the $\M$-norm (the total variation norm in the sense of measures), of which the structural properties of the solutions are still not completely understood, though a theory is beginning to emerge. The generality of these kinds of representer theorems has been especially useful in some of the recent development of the notion of \emph{reproducing kernel Banach spaces}~\citep{rkbs,rkbs-book} and of an infinite-dimensional theory of compressed sensing~\citep{infinite-dim-cs1, infinite-dim-cs2} as well as other inverse problems set in the continuous-domain~\citep{inv-prob-space-measures}.

We build off of these recent results, and propose a family of seminorms (indexed by an integer $m \geq 2$)
\begin{equation}
    \norm{f}_{(m)} \coloneqq c_d \norm{\partial_t^m \ramp^{d-1} \RadonOp f}_{\M(\cyl)},
    \label{eq:seminorms}
\end{equation}
where $\RadonOp$ is the Radon transform defined in \cref{eq:radon-transform}, $\Lambda^{d-1}$ is a ramp filter in the Radon domain defined in \cref{eq:ramp-filter}, $\partial_t^m$ is the $m$th partial derivative with respect to $t$, the offset variable in the Radon domain discussed in \cref{subsec:radon-transform}, and $c_d$ is a dimension dependent constant defined in \cref{eq:cd}, $\norm{\dummy}_{\M(\cyl)}$ denotes the total variation norm (in the sense of measures) on the Radon domain. We remark that the $\M$-norm can be viewed, as a ``generalization'' of the $L^1$-norm, with the key property that we can apply the $\M$-norm to distributions that are also ``absolutely integrable'' such as the Dirac impulse (i.e., distributions that can be associated with a finite Radon measure). The space $\cyl$ denotes the Radon domain; in particular, the Radon transform computes integrals over hyperplanes in $\R^d$. Since every hyperplane can be written as $\curly{\vec{x} \in \R^d \st \vec{\gamma}^\T\vec{x} = t}$ for $\vec{\gamma} \in \Sph^{d-1}$, the surface of the $\ell^2$-sphere in $\R^d$, and $t \in \R$, the Radon domain is $\cyl$. Finally, the space $\M(X)$ is the Banach space of finite Radon measures on $X$.

The family of seminorms in \cref{eq:seminorms} are thus exactly total variation seminorms in the Radon domain. For brevity, we will write
\[
    \ROp_m \coloneqq c_d\,\partial_t^m \ramp^{d-1} \RadonOp.
\]
Before stating our representer theorem, we remark that our result requires that the null space of the operator $\ROp_m$ is small, i.e., finite-dimensional. As discussed in~\cite{L-splines} and in the $L^2$-theory of radial basis functions and polyharmonic splines~\cite[Chapter~10]{scattered-data-approx}, constructing operators acting on multivariate functions with finite-dimensional null spaces is nearly impossible\footnote{For example, consider $\Delta$, the Laplacian operator in $\R^d$. Its null space is the space of harmonic functions which is infinite-dimensional for $d \geq 2$. On the other hand, the univariate Laplacian operator, $\d^2/\d x^2$, has a finite-dimensional null space which is simply $\spn\curly{1, x}$.}. To bypass this technicality, we use a common technique from variational spline theory (see, e.g.,~\cite{L-splines}) and impose a \emph{growth restriction} to the functions of interest via the weighted Lebesgue space $L^{\infty, n_0}(\R^d)$ (not to be confused with the Lorentz spaces), defined via the weighted $L^\infty$-norm
\[
    \norm{f}_{\infty, n_0} \coloneqq \esssup_{\vec{x} \in \R^d}\: \abs{f(\vec{x})}\paren{1 + \norm{\vec{x}}_2}^{-n_0},
\]
where $n_0 \in \mathbb{Z}$ is the \emph{algebraic growth rate}. In other words, the space $L^{\infty, n_0}(\R^d)$ is the space of functions mapping $\R^d \to \R$ with algebraic growth rate $n_0$. We will later see in \cref{eq:growth-restriction} that the appropriate choice of algebraic growth rate for the operator $\ROp_m$ is $n_0 \coloneqq m - 1$. This allows us to define the (growth restricted) \emph{null space} of $\ROp_m$ as
\begin{equation}
    \N_m \coloneqq \curly{q \in L^{\infty, m - 1}(\R^d) \st \ROp_m q = 0}
    \label{eq:null-space}
\end{equation}
and the (growth restricted) \emph{native space} of $\ROp_m$ as
\begin{equation}
    \F_m \coloneqq \curly{f \in L^{\infty, m - 1}(\R^d) \st \ROp_m f \in \M(\cyl)}.
    \label{eq:native-space}
\end{equation}
We prove in \cref{lemma:finite-dim-null-space} that $\N_m$ is indeed finite-dimensional.

We now state our representer theorem, in which we show that there exists a sparse solution to the inverse problem in \cref{eq:generic-inverse-problem} when the seminorm takes the form in \cref{eq:seminorms} and the search space is $\F_m$. In particular, we show that that the sparse solution takes the form of the sum of a single-hidden layer neural network as in \cref{eq:nn-as-superposition} and a low degree polynomial. In this context, sparse means that the width of the neural network and the degree of the polynomial are, a priori, bounded from above.
\begin{theorem} \label{thm:rep-thm}
    Assume the following:
    \begin{enumerate}[label=\arabic*.]
        \item The function $G: \R^N \to \R$ is a strictly convex, coercive, and lower semi-continuous.
        \item The operator $\sensing: \F_m \to \R^N$ is continuous\footnote{In order to define continuity, $\F_m$ needs to be a \emph{topological} vector space. We prove in \cref{thm:banach-space} that $\F_m$ is a Banach space, which provides it a topology, allowing  continuity to be defined.}, linear, and surjective.
        \item The inverse problem is well-posed over the null space $\N_m$ of $\ROp_m$, i.e., $\sensing q_1 = \sensing q_2$ if and only if $q_1 = q_2$, for any $q_1, q_2 \in \N_m$. \label{item:well-posed-null-space}
    \end{enumerate}
    Then, there exists a sparse minimizer to the variational problem
    \begin{equation}
        \min_{f \in \native_m} \: \datafit(\sensing f) + \norm{\ROp_m f}_{\M(\cyl)}
        \label{eq:inverse-problem}
    \end{equation}
    that takes the form
    \begin{equation}
        s(\vec{x}) = \sum_{k=1}^K v_k \, \rho_m(\vec{w}_k^\T \vec{x} - b_k) + c(\vec{x}),
        \label{eq:ridge-spline}
    \end{equation}
    where $K \leq N - \dim \N_m$,
    $\rho_m = \max\curly{0, \dummy}^{m-1} / (m - 1)!$, $\vec{w}_k \in \Sph^{d-1}$, $v_k \in \R$, $b_k \in \R$, and $c(\dummy)$ is a polynomial of degree strictly less than $m$.
\end{theorem}

Proving \cref{thm:rep-thm} hinges on several technical results, the most important being the topological structure of the native space $\native_m$. In order to do any kind of analysis (e.g., proving that minimizers of \cref{eq:inverse-problem} even exist), we require the native space $\native_m$ to have some ``nice'' topological structure. We prove in \cref{thm:banach-space} that $\native_m$, when equipped with a proper direct-sum topology, is a Banach space. This key result hinges on being able to construct a stable right inverse of the operator $\ROp_m$, which we outline in \cref{lemma:right-inverse}. We remark that exhibiting a Banach space structure of the native space of an operator is common in variational inverse problems, e.g., in the theory of $\Ell$-splines~\citep{L-splines, native-banach}. We do remark, however, our result is, to the best of our knowledge, the first time exhibiting this structure on a non-Euclidean domain, which causes some nuances compared to prior work of \citet{L-splines, native-banach}.

\Cref{thm:rep-thm} shows that while the problem is posed in the continuum, it admits \emph{parametric} solutions in terms of a finite number of parameters. This demonstrates the sparsifying effect of the $\M$-norm, similar to its discrete analogue, the $\ell^1$-norm. We also remark that although the problem in \cref{thm:rep-thm} admits a sparse solution, it is important to note that the solution may not be unique and there may also exist non-sparse solutions.

\begin{remark}
    The polynomial term $c(\vec{x})$ that appears in \cref{eq:ridge-spline} corresponds to a term in the null space $\N_m$. When $m = 2$, the network in \cref{eq:ridge-spline} is a ReLU network and $c(\vec{x})$ takes the form
    \[
        c(\vec{x}) = \vec{u}^\T\vec{x} + s,
    \]
    where $\vec{u} \in \R^d$ and $s \in \R$. Thus, when $m = 2$, \cref{eq:ridge-spline} corresponds to a ReLU network with a \emph{skip connection}~\citep{skip-connections}.
\end{remark}

\begin{remark} \label[Remark]{rem:rescale}
    The fact that $\vec{w}_k \in \Sph^{d-1}$ does not restrict the single-hidden layer neural network due to the homogeneity of the truncated power functions. Indeed, given any single-hidden layer neural network with $\vec{w}_k \in \R^d \setminus \curly{\vec{0}}$, we can use the fact that $\rho_m$ is homogeneous of degree $m - 1$ to rewrite the network as
    \[
        \vec{x} \mapsto \sum_{k=1}^K v_k \norm{\vec{w}_k}_2^{m - 1} \rho_m(\tilde{\vec{w}}_k^\T \vec{x} - \tilde{b}_k) + c(\vec{x}),
    \]
    where $\tilde{\vec{w}}_k \coloneqq \vec{w}_k / \norm{\vec{w}_k}_2 \in \Sph^{d-1}$ and $\tilde{b}_k \coloneqq b_k / \norm{\vec{w}_k}_2 \in \R$. We use this fact to prove \cref{prop:equiv-opts} which recasts the variational problem in \cref{eq:inverse-problem} as a finite-dimensional neural network training problem (with no constraints on the input layer weights).
\end{remark}
The proof of \cref{thm:rep-thm} appears in \cref{sec:rep-thm}.

\subsection{Ridge splines} Splines and variational problems are tightly connected~\citep{splines-sobolev-seminorm, splines-variational, L-splines}. In the framework of $\Ell$-splines~\citep{L-splines}, a pseudodifferential operator, $\Ell: \Sch'(\R^d) \to \Sch'(\R^d)$, where $\Sch'(\R^d)$ denotes the space of tempered distributions on $\R^d$, is associated with a spline, and variational problems of the form
\begin{equation}
    \min_{f \in \native_m} \: \datafit(\sensing f) + \norm{\Ell f}_{\M(\R^d)}
    \label{eq:L-spline-problem}
\end{equation}
are studied, where $\datafit$ is a data fitting term, $\sensing$ is a measurement operator, $\native_m$ is the native space of $\Ell$, and $\M(\R^d)$ is the space of finite Radon measures on $\R^d$ which forms a Banach space when equipped with $\norm{\dummy}_{\M(\R^d)}$, the total variation norm in the sense of measures. The key result from~\cite{L-splines} is a representer theorem for the above problem which states that there exists a sparse solution which is a so-called $\Ell$-spline. By associating an operator to a spline, we have a simple way to characterize what functions are splines via the following definition.
\begin{definition}[nonuniform $\Ell$-spline {\cite[Definition~2]{L-splines}}] \label[Definition]{defn:L-spline}
    A function $s: \mathbb{R}^d \to \mathbb{R}$ (of slow growth) is said to be a \emph{nonuniform $\Ell$-spline} if
    \[
    \Ell\curly{s} = \sum_{k=1}^K v_k \, \delta_{\R^d}(\dummy - \vec{x}_k),
    \]
    where $\delta_{\R^d}$ denotes the Dirac impulse on $\R^d$, $\curly{v_k}_{k=1}^K$ is a sequence of weights and the locations of Dirac impulses are at the spline knots $\curly{\vec{x}_k}_{k=1}^K$.
\end{definition}

Due to the similarities between the variational problem in \cref{eq:L-spline-problem} and our variational problem in \cref{eq:inverse-problem}, we can similarly define the notion of a (polynomial) \emph{ridge spline}. Before stating this definition, we remark that in this paper we will be working with Dirac impulses on different domains. For clarity, we will subscript the ``$\delta$'' with the appropriate domain, e.g., $\delta_{\R^d}$ denotes the Dirac impulse on $\R^d$ and $\delta_{\cyl}$ denotes the Dirac impulse on $\cyl$.
\begin{definition}[nonuniform polynomial ridge spline] \label[Definition]{defn:ridge-spline}
    A function $s: \mathbb{R}^d \to \mathbb{R}$ (of slow growth) is said to be a \emph{nonuniform polynomial ridge spline} of order $m$ if
    \begin{equation}
      \ROp_m\curly{s} = \sum_{k=1}^K v_k \, \sq{\frac{\delta_\cyl(\dummy - \vec{z}_k) + (-1)^m \delta_\cyl(\dummy + \vec{z}_k)}{2}},
      \label{eq:ridge-spline-innovation}
    \end{equation}
    where $\curly{v_k}_{k=1}^K$ is a sequence of weights and the locations of the Dirac impulses are at $\vec{z}_k = (\vec{w}_k, b_k) \in \cyl$. The collection $\curly{\vec{z}_k}_{k=1}^K$ can be viewed as a collection of Radon domain spline knots.
\end{definition}

\begin{remark} \label[Remark]{rem:radon-impulse}
    The reason
    \begin{equation}
        \frac{\delta_\cyl(\dummy - \vec{z}_k) + (-1)^m \delta_\cyl(\dummy + \vec{z}_k)}{2}
        \label{eq:radon-impulse}
    \end{equation}
    appears in \cref{eq:ridge-spline-innovation} rather than $\delta_\cyl(\dummy - \vec{z}_k)$ is due to the fact that the operator $\ROp_m$ maps functions $f \in \F_m$ to even (respectively odd) elements of $\M(\cyl)$ when $m$ is even (respectively odd). Thus, \cref{eq:radon-impulse} can be viewed as an even or odd version of the normal translated Dirac impulse in the sense that when acting on even or odd test functions defined on $\cyl$, it is the point evaluation operator.
\end{remark}
    
\begin{remark} \label[Remark]{rem:nn-atoms-sparsified}
    When $s$ is a neural network as in \cref{eq:ridge-spline}, we have that \cref{eq:ridge-spline-innovation} holds. The way to understand this is that the neurons in \cref{eq:ridge-function} are ``sparsified'' by $\ROp_m$ in the sense that
    \[
        \ROp_m r_{(\vec{w}, b)}^{(m)} = \frac{\delta_\cyl(\dummy - \vec{z}_k) + (-1)^m \delta_\cyl(\dummy + \vec{z}_k)}{2},
    \]
    where $r_{(\vec{w}, b)}^{(m)}(\vec{x}) \coloneqq \rho_m(\vec{w}^\T\vec{x} - b)$, $(\vec{w}, b) \in \cyl$. We show that this is true in \cref{lemma:nn-atoms-sparsified}. In other words, $r_{(\vec{w}, b)}^{(m)}$ can be viewed as a translated \emph{Green's function} of $\ROp_m$, where the translation is in the Radon domain.
\end{remark}

\begin{figure}[htb!]
    \centering
    \begin{tikzpicture}
        \node (a) at (0, 0) {
            \includegraphics[scale=0.4]{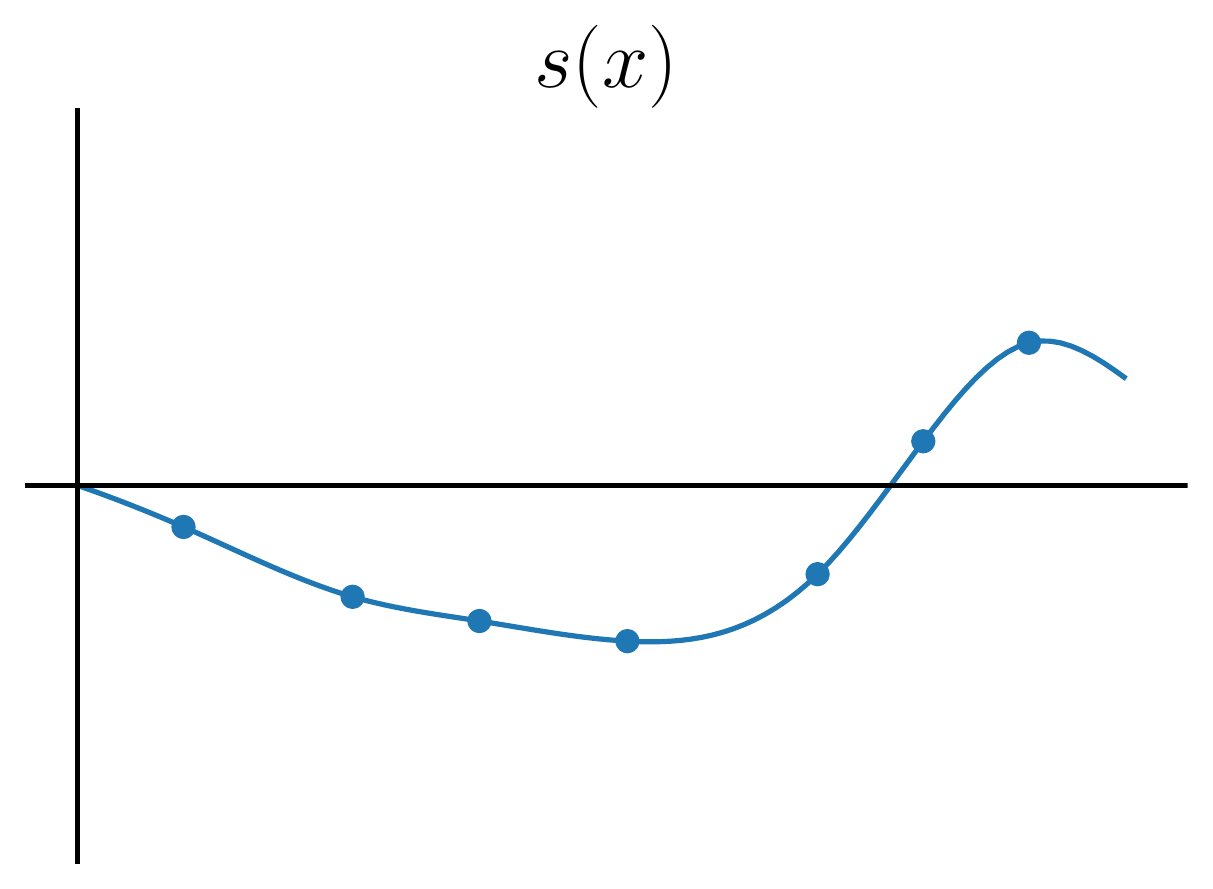}
        };
        \node (b) at (8, 0) {
            \includegraphics[scale=0.4]{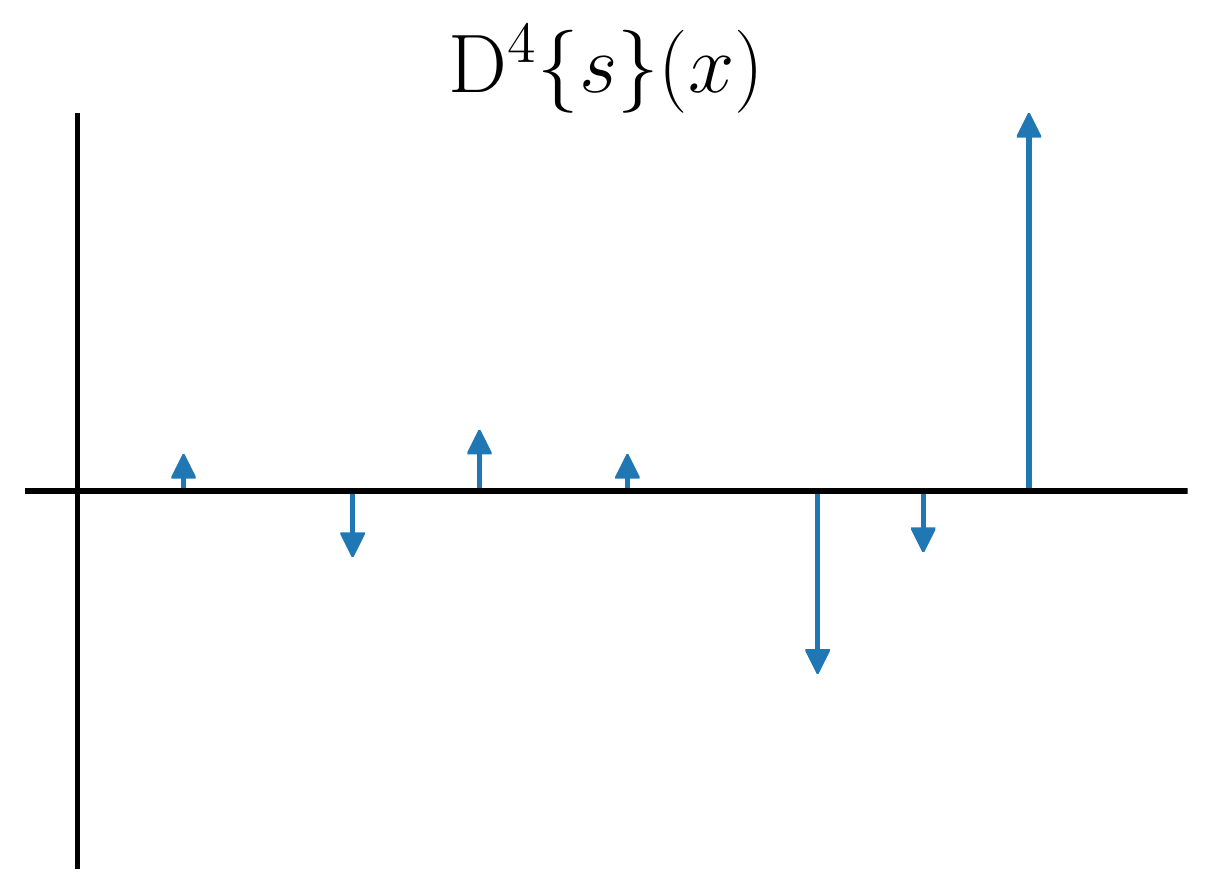}
        };
        \draw[gray, -{Latex[width=2mm]}] (a) -- (b) node[midway, above]{\footnotesize$\D^4$};
    \end{tikzpicture}
    \caption{In the left plot we have a cubic spline with 7 knots. After applying $\D^4$, the fourth derivative operator, we are left with 7 Dirac impulses as seen in the right plot.}
    \label{fig:cubic-spline}
\end{figure}

\begin{figure}[htb!]
    \centering
    \begin{tikzpicture}
        \node (a) at (0, 0) {
            \includegraphics[scale=0.4]{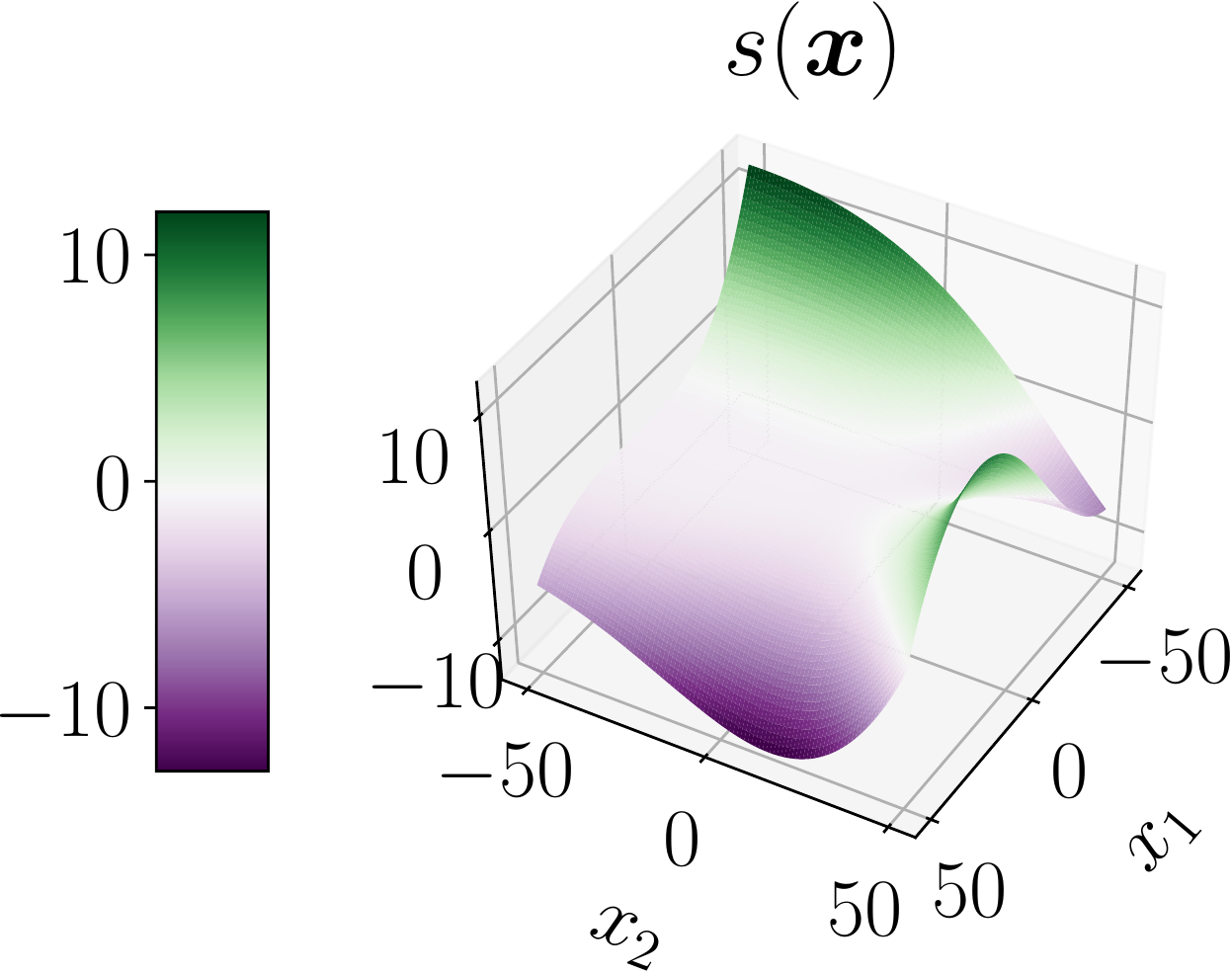}
        };
        \node (c) at (8, 0) {
            \includegraphics[scale=0.4]{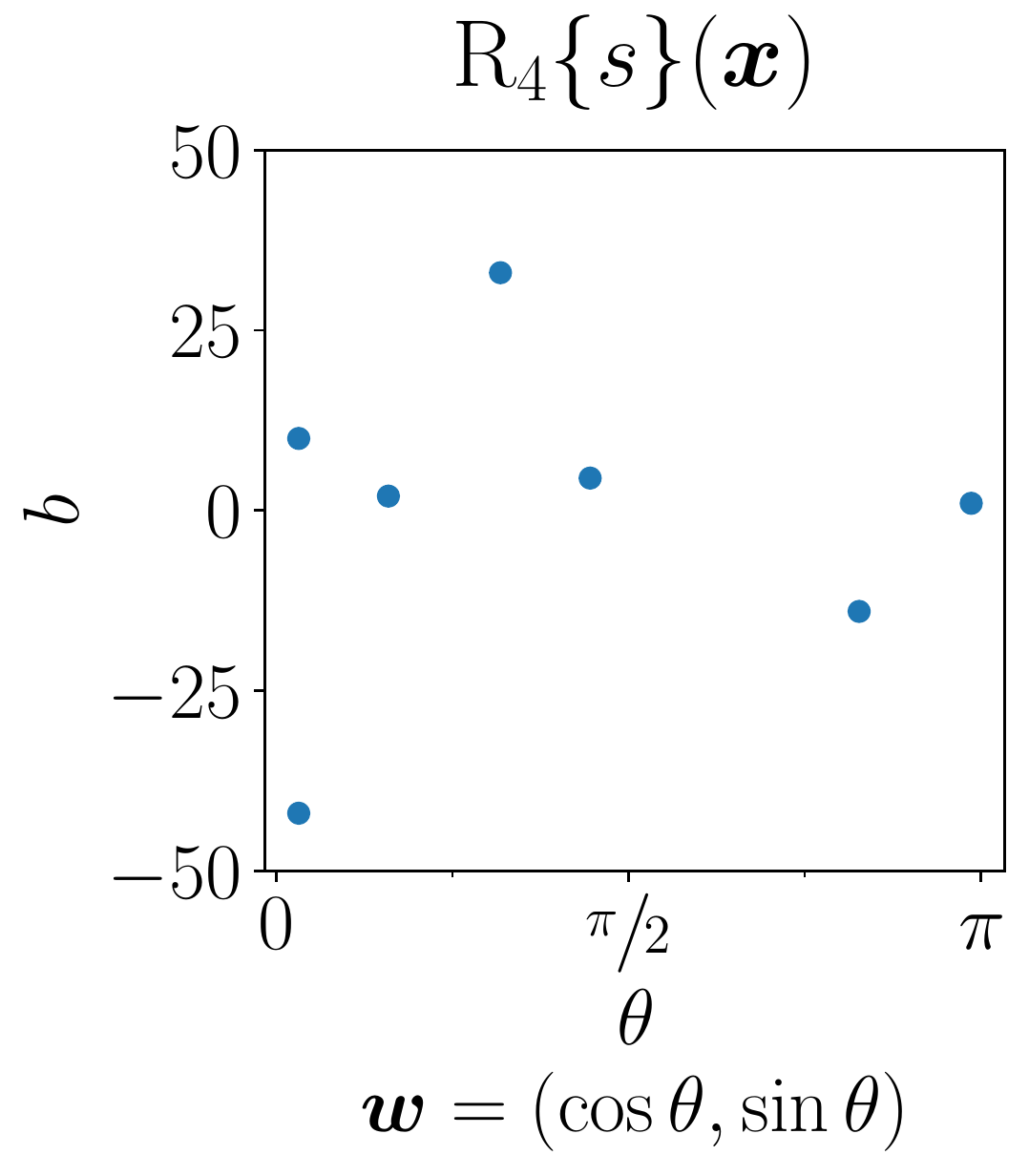}
        };
        \draw[gray, -{Latex[width=2mm]}] (a) -- (c) node[midway, above]{\footnotesize$\ROp_4$};
    \end{tikzpicture}
    \caption{In the left plot we have a two-dimensional cubic ridge spline with 7 neurons. After applying $\ROp_4$, we are left with 7 Dirac impulses in the Radon domain, which are designated by the dots in the right plot. We have parameterized the directions in the Radon domain by $\theta \in [0, \pi)$. This parameterization of the two-dimensional Radon domain is known as a \emph{sinogram}. This parameterization of the Radon domain eliminates the two impulses per neuron we see in \cref{eq:ridge-spline-innovation} as $\theta \in [0, \pi)$ is only ``half'' of the unit circle $\Sph^1$.}.
    \label{fig:cubic-ridge-spline}
\end{figure}

We illustrate the sparsifying effect of the operator $\Ell$ in the case of cubic splines, i.e., $\Ell = \D^4$, the fourth derivative operator, in \cref{fig:cubic-spline}. We also illustrate the sparsifying effect of the operator $\ROp_m$ in the case of cubic ridge splines, i.e., $m = 4$, in \cref{fig:cubic-ridge-spline}. We also remark that  in the univariate case ($d = 1$), our notion of a polynomial ridge spline of order $m$ exactly coincides with the classical notion of a univariate polynomial spline of order $m$. We show this in \cref{subsec:1D-splines}.

We finally remark that when $m$ is even, we have the equality
\[
    \ROp_m = c_d\,\ramp^{d-1} \RadonOp \Delta^{m/2},
\]
by the intertwining relations of the Radon transform and the Laplacian, which we later discuss in \cref{eq:radon-intertwining}. This provides another way to understand how $\ROp_m$ sparsifies ridge splines. We illustrate this in the $m = 2$ (i.e., ReLU network) case in \cref{fig:linear-ridge-spline}.

\begin{figure}[htb!]
    \centering
    \begin{tikzpicture}
        \node (a) at (0, 0) {
            \includegraphics[scale=0.33]{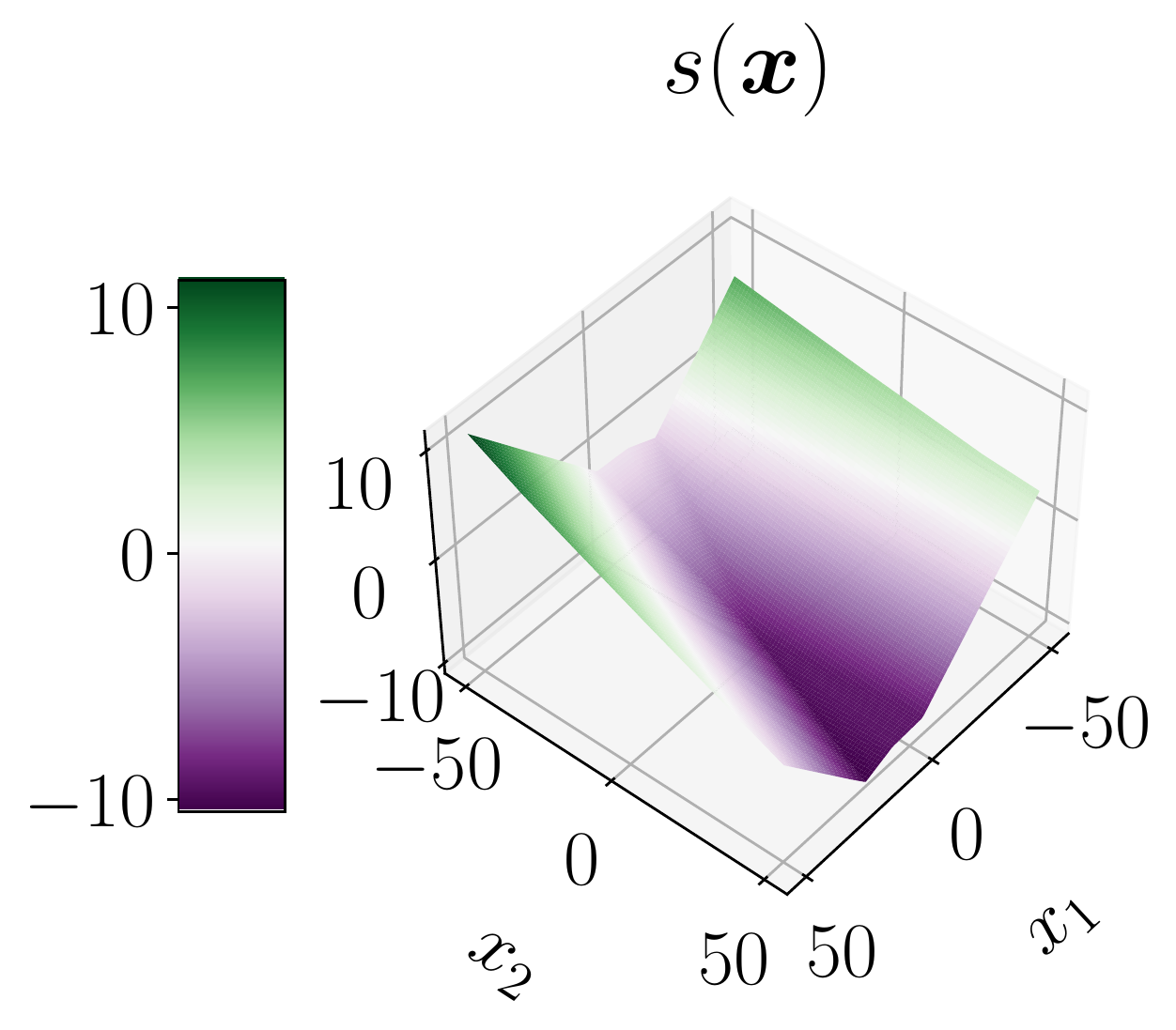}
        };
        \node (b) at (5, 0) {
            \includegraphics[scale=0.33]{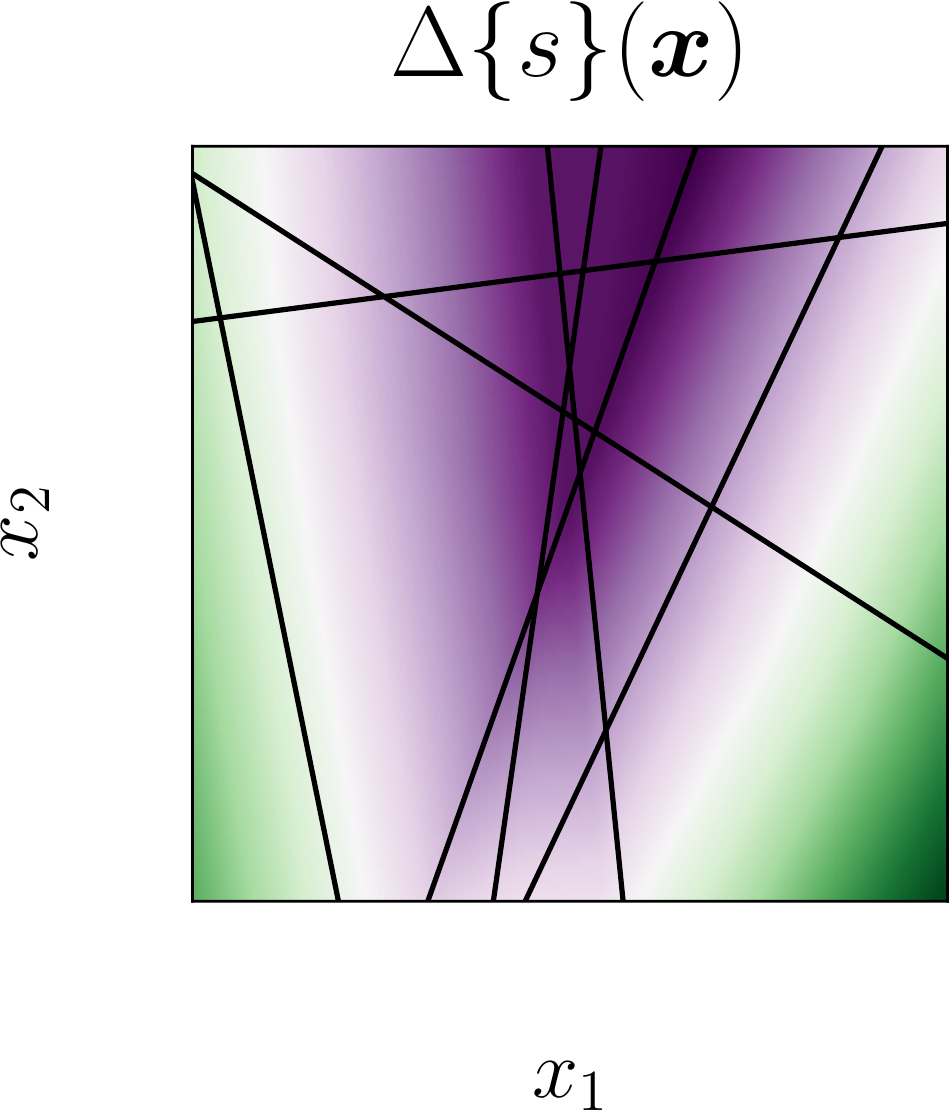}
        };
        \node (c) at (10, 0) {
            \includegraphics[scale=0.33]{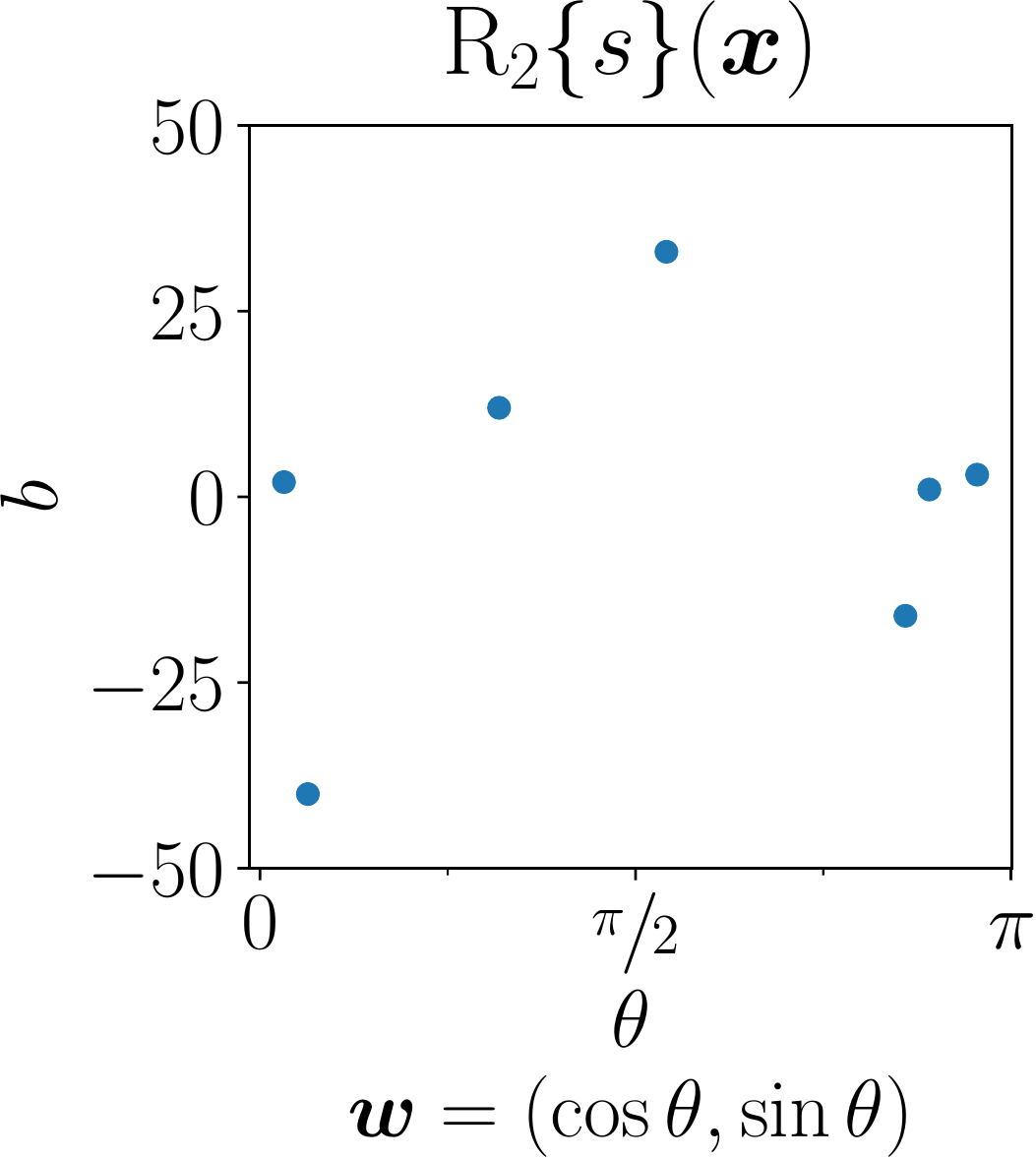}
        };
        \draw[gray, -{Latex[width=2mm]}] (a) -- (b) node[midway, above]{\footnotesize$\Delta$};
        \draw[gray, -{Latex[width=2mm]}] (b) -- (c) node[midway, above]{\footnotesize$c_d \ramp \RadonOp$};
    \end{tikzpicture}
    \caption{On the left plot we have a two-dimensional linear ridge spline (single-hidden layer ReLU network) with 7 neurons. After applying $\Delta$, we get an ``impulse sheet'', i.e., a mapping of the form $\vec{x} \mapsto \delta_\R(\vec{w}_k^\T\vec{x} - b_k)$, for each neuron, designated by the black lines in the top down view of the linear ridge spline in the middle plot. Then, after applying the Radon transform and ramp filter to the middle plot, we arrive with 7 Dirac impulses in the Radon domain, which are designated by the dots in the left plot. Just as in \cref{fig:cubic-ridge-spline}, we have 
    parameterized the directions in the Radon domain by $\theta \in [0, \pi)$, eliminating the two impulses per neuron we see in \cref{eq:ridge-spline-innovation}.}
    \label{fig:linear-ridge-spline}
\end{figure}
    
\subsection{Scattered data approximation and neural network training}
Since \cref{thm:rep-thm} says that a single-hidden layer neural network as in \cref{eq:ridge-spline} is a solution to the continuous-domain inverse problem in \cref{eq:inverse-problem}, we can recast the continuous-domain problem in \cref{eq:inverse-problem} as the \emph{finite-dimensional neural network training} problem
\begin{equation}
    \min_{\vec{\theta} \in \Theta} \: \datafit(\sensing f_\vec{\theta}) + \norm{\ROp_m f_\vec{\theta}}_{\M(\cyl)},
    \label{eq:nn-problem}
\end{equation}
so long as the number of neurons $K$ is large enough\footnote{We characterize what large enough means in \cref{prop:equiv-opts}.} ($K \geq N$ suffices, giving insight into the efficacy of overparameterization in neural network models), where
\[
    f_\vec{\theta}(\vec{x}) \coloneqq \sum_{k=1}^K v_k \, \rho_m(\vec{w}_k^\T \vec{x} - b_k) + c(\vec{x}),
\]
where $\vec{\theta} = (\vec{w}_1, \ldots, \vec{w}_K, v_1, \ldots, v_K, b_1, \ldots, b_K, c)$ contains the neural network parameters and $\Theta$ is the collection of all $\vec{\theta}$ such that $v_k \in \R$, $\vec{w}_k \in \R^d$, and $b_k \in \R$ for $k = 1, \ldots, K$, and where $c$ is a polynomial of degree strictly less than $m$. We show in \cref{lemma:nn-norm} that
\[
    \norm{\ROp_m f_\vec{\theta}}_{\M(\cyl)} = \sum_{k=1}^K \abs{v_k} \norm{\vec{w}_k}_2^{m-1},
\]
and then use this fact to show that \cref{eq:nn-problem} is equivalent to two neural network training problems, with variants of well-known neural network regularizers, in the following proposition.
\begin{theorem}
    \label{prop:equiv-opts}
    The solutions to the finite-dimensional optimization in \cref{eq:nn-problem} are solutions to optimization in \cref{eq:inverse-problem} so long as $K \geq N - \dim \N_m$. Additionally, the optimization in \cref{eq:nn-problem} is equivalent to
    \begin{equation}
        \min_{\vec{\theta} \in \Theta} \:  \datafit(\sensing f_\vec{\theta}) + {\sum_{k=1}^K \abs{v_k} \norm{\vec{w}_k}_2^{m - 1}},
        \label{eq:nn-training-with-pathnorm}
    \end{equation}
    for any $K \in \mathbb{N}$. Furthermore, the solutions to
    \begin{equation}
        \min_{\vec{\theta} \in \Theta} \: \datafit(\sensing f_\vec{\theta}) + {\frac{1}{2}\sum_{k=1}^K \paren{\abs{v_k}^2 + \norm{\vec{w}_k}_2^{2m - 2}}}
        \label{eq:nn-training-with-weight-decay}
    \end{equation}
    are also solutions to \cref{eq:nn-training-with-pathnorm} for any $K \in \mathbb{N}$. Finally, for both problems in \cref{eq:nn-training-with-pathnorm,eq:nn-training-with-weight-decay}, for $K_1$ and $K_2$ such that $K_1 > K_2$, a global minimizer when $K = K_2$ will always be a global minimizer when $K = K_1$.
\end{theorem}

When $m = 2$, which coincides with neural networks with ReLU activation functions, \cref{eq:nn-training-with-pathnorm,eq:nn-training-with-weight-decay} correspond to previously studied training problems. The regularizer in \cref{eq:nn-training-with-pathnorm} coincides with the notion of the \emph{$\ell^1$-path-norm} regularization as proposed in~\cite{path-norm} and the regularizer in \cref{eq:nn-training-with-weight-decay} coincides with the notion of training a neural network with \emph{weight decay} as proposed in~\cite{weight-decay}. Thus, our result shows that these notions of regularization are intrinsically tied to the ReLU activation function, and, perhaps, variants such as the regularizers that appear in \cref{eq:nn-training-with-pathnorm,eq:nn-training-with-weight-decay} should be used in practice for non-ReLU activation functions, where $m - 1$ could corresponds to the algebraic growth rate of the activation function.
    
In machine learning, the measurement model is usually \emph{ideal sampling}, i.e., the measurement operator $\sensing$ acts on a function $f: \R^d \to \R$ via
\begin{equation}
    \sensing: \F_m \ni f \mapsto \begin{bmatrix}
        \ang{\delta_{\R^d}(\dummy - \vec{x}_1), f} \\
        \vdots \\
        \ang{\delta_{\R^d}(\dummy - \vec{x}_N), f}
    \end{bmatrix}
    =
    \begin{bmatrix}
        f(\vec{x}_1) \\
        \vdots \\
        f(\vec{x}_N)
    \end{bmatrix} \in \R^N,
    \label{eq:ideal-sampling}
\end{equation}
so the problem is to approximate the scattered data $\curly{(\vec{x}_n, y_n)}_{n=1}^N \subset \R^d \times \R$. For the above $\sensing$ to be a valid measurement operator for \cref{thm:rep-thm}, it must be continuous.
\begin{lemma} \label[Lemma]{lemma:ideal-sampling-continuous}
    The operator $\sensing: \F_m \to \R^N$ defined in \cref{eq:ideal-sampling} is continuous.
\end{lemma}
The proof of \cref{lemma:ideal-sampling-continuous} appears in \cref{app:aux-proofs}. \Cref{lemma:ideal-sampling-continuous} also says that $\F_m$ is a \emph{reproducing kernel Banach space}.

By choosing an appropriate data fitting term $G$, the generality of our main result in \cref{thm:rep-thm} says that the solutions problems with interpolation constraints
\[
    \min_{f \in \native_m} \: \norm{\ROp_m f}_{\M(\cyl)} \quad\subj\quad f(\vec{x}_n) = y_n, \: n = 1, \ldots, N
\]
and to regularized problems where we have soft-constraints in the case of noisy data
\begin{equation}
    \min_{f \in \native_m} \: \sum_{n=1}^N \ell(f(\vec{x}_n), y_n) + \lambda \norm{\ROp_m f}_{\M(\cyl)},
    \label{eq:regularized-problem}
\end{equation}
where $\lambda > 0$ is an adjustable regularization parameter and $\ell(\dummy, \dummy)$ is an appropriate loss function, e.g., the squared error loss, are single-hidden layer neural networks. We can then invoke \cref{prop:equiv-opts} to recast the problem in \cref{eq:regularized-problem} with either of the equivalent finite-dimensional neural network training problems:
\begin{align*}
    &\min_{\vec{\theta} \in \Theta} \: \sum_{n=1}^N \ell(f_\vec{\theta}(\vec{x}_n), y_n) + \lambda 
    \sum_{k=1}^K \abs{v_k} \norm{\vec{w}_k}_2^{m - 1} \numberthis \label{eq:tikhonov-problem} \\
    &\min_{\vec{\theta} \in \Theta} \: \sum_{n=1}^N \ell(f_\vec{\theta}(\vec{x}_n), y_n) + \frac{\lambda}{2}\sum_{k=1}^K \paren{\abs{v_k}^2 + \norm{\vec{w}_k}_2^{2m - 2}},
\end{align*}
so long as the number of neurons $K$ is large enough as stated in \cref{prop:equiv-opts}. The two problems in the above display correspond to how neural network training problems are actually set up.

\subsection{Statistical generalization bounds}
Neural networks are widely used for pattern classification.
In the ideal sampling scenario, the generality of \cref{thm:rep-thm} allows us to consider optimizations of the form
\begin{equation}
    \min_{f \in \F_m} \: \sum_{n=1}^N \ell\big(y_n f(\vec{x}_n)\big) \quad\subj\quad \norm{\ROp_m f}_{\M(\cyl)} \leq B,
    \label{eq:class-opt}
\end{equation}
for some constant $B < \infty$, where $\ell(\dummy)$ is an appropriate $L$-Lipschitz loss function of the product $y_n f(\vec{x}_n)$. If we assume that $\curly{(\vec{x}_n,y_n)}_{n=1}^N$ are drawn independently and identically from some unknown underlying probability distribution, $y_n \in \curly{-1,+1}$, $n=1,\dots,N$, and the loss function assigns positive losses when $\sgn(f(\vec{x}_n)) \neq y_n$ (or equivalently when $y_n f(\vec{x}_n) < 0$), this is the \emph{binary classification} setting. Given this set up, it is natural to examine if solutions to \cref{eq:class-opt} predict well on new random examples $(\vec{x},y)$ drawn independently from the same underlying distribution.

We can invoke \cref{prop:equiv-opts} and consider optimization over neural network parameters by considering the recast optimization to \cref{eq:class-opt}
\begin{equation}
    \min_{\vec{\theta} \in \Theta} \: \sum_{n=1}^N \ell\big(y_n f_\vec{\theta}(\vec{x}_n)\big) \quad\subj\quad  \sum_{k=1}^K \abs{v_k} \norm{\vec{w}_k}_2^{m-1} \leq B,
    \label{eq:nn-class-opt}
\end{equation}
where
\[
    f_\vec{\theta}(\vec{x}) \coloneqq \sum_{k=1}^K v_k \rho_m(\vec{w}_k^\T \vec{x} - b_k) + c(\vec{x}).
\]
In particular, the solution to \cref{eq:nn-class-opt} is known as an \emph{Ivanov estimator}~\citep{estimators} which is equivalent to the solution to \cref{eq:tikhonov-problem} for a particular choice of $\lambda$ which depends on $B$ and the data through the data fitting term. In this section we provide a \emph{generalization bound} for the Ivanov estimator.

Let $\bar{f}$ be a minimizer of the optimization in \cref{eq:nn-class-opt}. We show that $B$ directly controls the error probability of $\bar{f}$, i.e., ${\mathbb{P}}\big(y \bar{f}(\vec{x})<0\big)$, where $(\vec{x},y)$ is an independent sample from the underlying distribution.  This is referred to as the \emph{generalization error} in machine learning parlance. We follow the standard approach based on Rademacher complexity \citep{bartlett2002rademacher,shalev2014understanding}.

Let $\F$ be a hypothesis space. For every $f \in \F$, define its \emph{risk} and \emph{empirical risk}
\[
    R(f) \coloneqq \E\left[\ell\big(y f(\vec{x})\big)\right]
    \quad\text{and}\quad
    \hat{R}_N(f) \coloneqq \frac{1}{N} \sum_{n=1}^N \ell\big(y_n f(\vec{x}_n)\big),
\]
and assume the loss function satisfies $0 \leq \ell\big(y_n f(\vec{x}_n)\big) \leq C_0$ almost surely, for $n=1,\ldots,N$ and some constant $C_0 < \infty$. Then, for every $f \in \F$, we have the following generalization bound. With probability at least $1 - \delta$,
\[
    R(f) \leq \hat{R}_N(f) + L\, \Rad(\F) + C_0 \sqrt{\frac{\log(1 / \delta)}{2N}},
\]
where we use the fact that the loss $\ell$ is $L$-Lipschitz and $\Rad(\F)$ is the \emph{Rademacher complexity} of the class $\F$ defined via
\begin{equation}
    \Rad(\F) \coloneqq 2 \E\left[\sup_{f\in\F} \frac{1}{N} \sum_{n=1}^N\sigma_n f(\vec{x}_n)  \right],
    \label{eq:rad}
\end{equation}
where $\curly{\sigma_n}_{n=1}^N$ are independent and identically distributed Rademacher random variables. In particular, if the expected loss is an upper bound on the probability of error (e.g., squared error, $(1-yf(\vec{x}))^2$, or hinge loss, $\max\{0,1-yf(\vec{x})\}$), then we may use this to bound the probability of error $\P\left(y \bar{f}(\vec{x}) < 0\right) \leq R(\bar{f})$.

To provide a generalization bound for the minimizer $\bar{f}$ of \cref{eq:class-opt}, we assume that the empirical data satisfies $\norm{\vec{x}_n}_2 \leq C/2$ almost surely for $n=1,\dots,N$ and some constant $C < \infty$ and consider the hypothesis space
\[
    \F_{\Theta} \coloneqq \curly{f_\vec{\theta} \st \vec{\theta} \in \Theta, \:\: \sum_{k=1}^K \abs{v_k} \norm{\vec{w}_k}_2^{m-1} \leq B, \:\: \abs{b_k} \leq \frac{C}{2}, k = 1, \ldots, K, \:\: K \geq 0}.
\]
The reason we impose that $\abs{b_k} \leq C/2$, $k = 1, \ldots, K$, is because we will later show in \cref{lemma:bias-bound} that all solutions to \cref{eq:class-opt} satisfy 
\[
    \abs{b_k} \leq \max_{n=1, \ldots, N} \norm{\vec{x}_n}_2.
\]
for every $k = 1, \ldots, K$. We bound the Rademacher complexity of $\F_\Theta$ in the following theorem.
\begin{theorem}
  Assume that $\norm{\vec{x}_n}_2 \leq C/2$ almost surely for $n=1,\dots,N$ and some constant $C < \infty$. Then,
  \[
    \Rad(\F_\Theta) \leq \frac{2B C^{m-1}}{\sqrt{N} (m-1)!} + \Rad(c),
  \]
  where $\Rad(c)$ denotes the Rademacher complexity of the polynomial terms $c(\vec{x})$ that appear in the solutions to the optimization in \cref{eq:nn-class-opt}.
\label{thm:rad}
\end{theorem}
  
\begin{remark}
    \cref{thm:rad} shows that the Rademacher complexity, and hence the generalization error, is controlled by bounding the seminorm  $\norm{\ROp_m f}_{\M(\cyl)} \leq B$.  In practice, neural networks are typically implemented without the polynomial term $c(\vec{x})$, in which case the same bound holds without $\Rad(c)$.
\end{remark}

 \section{Preliminaries \& Notation} \label{sec:prelim}
In this section we will introduce the mathematical formulation and notation used in the remainder of the paper.
\subsection{Spaces of functions and distributions}
\label{sec:dist}
Let $\Sch(\R^d)$ be the Schwartz space of smooth and rapidly
decaying test functions on $\R^d$.  Its continuous dual,
$\Sch'(\R^d)$, is the space of tempered distributions on
$\R^d$. Since we are interested in the Radon domain, we are also interested in these spaces on $\cyl$.  We say $\psi \in \Sch(\cyl)$ when $\psi$ is smooth and satisfies the decay condition~\cite[Chapter~6]{fourier}
\[
  \sup_{\substack{\vec{\gamma} \in \Sph^{d-1} \\ t \in \R}}
  \abs{\paren{1 +
  \abs{t}^k} \de[^\ell]{t^\ell} (\D\psi)(\vec{\gamma}, t)} < \infty
\]
for all integers $k, \ell \geq 0$ and for all differential operators $\D$ in $\vec{\gamma}$. Since the Schwartz spaces are nuclear, it follows that the above definition is equivalent to saying $\Sch(\cyl) = \mathcal{D}(\Sph^{d-1}) \,\hat{\otimes}\, \Sch(\R)$, where $\mathcal{D}(\Sph^{d-1})$ is the space of smooth functions on $\Sph^{d-1}$ and $\hat{\otimes}$ is the \emph{topological} tensor product~\cite[Chapter~III]{tvs}. We can then define the space of tempered distributions on $\cyl$ as its continuous dual, $\Sch'(\cyl)$.

We will later see in \cref{subsec:radon-transform} that in order to define the Radon transform of distributions, we will be interested in the \emph{Lizorkin test functions} $\Sch_0(\R^d)$ of highly time-frequency localized functions over $\R^d$~\citep{wavelets-lizorkin}. This is a closed subspace of $\Sch(\R^d)$ consisting of functions with all moments equal to $0$, i.e., $\varphi \in \Sch_0(\R^d)$ when $\varphi \in \Sch(\R^d)$ and
\[
  \int_{\R^d} \vec{x}^\vec{\alpha} \varphi(\vec{x}) \dd \vec{x} = 0,
\]
for every multi-index $\vec{\alpha}$. We can then define the space of \emph{Lizorkin
distributions}, $\Sch_0'(\R^d)$, the continuous dual of the
Lizorkin test functions. The space of Lizorkin distributions can be viewed as
being topologically isomorphic to the quotient space of tempered distributions
by the space of polynomials, i.e., if $\mathcal{P}(\R^d)$ is the space
of polynomials on $\R^d$, then $\Sch_0'(\R^d) \cong
\Sch'(\R^d) /
\mathcal{P}(\R^d)$~\cite[Chapter~8]{lizorkin-triebel}.  Just as above, we can define the Lizorkin test functions on $\cyl$ as $\Sch_0(\cyl) =
\mathcal{D}(\Sph^{d-1}) \,\hat{\otimes}\, \Sch_0(\R)$ and
the space of Lizorkin distributions on $\cyl$
as its continuous dual, $\Sch_0'(\cyl)$.

Let $X$ be a locally compact Hausdorff space. The Riesz--Markov--Kakutani representation theorem says that $\M(X)$, the space of finite Radon measures on $X$, is the continuous dual of $C_0(X)$, the space of continuous functions vanishing at infinity~\cite[Chapter~7]{folland}. Since $C_0(X)$ is a Banach space when equipped with the uniform norm, we have
\begin{equation}
  \norm{u}_{\M(X)}
  \coloneqq \sup_{\substack{\varphi \in C_0(X) \\
  \norm{\varphi}_\infty = 1}} \ang{u, \varphi}.
\label{eq:M-norm}
\end{equation}

The norm $\norm{\dummy}_{\M(X)}$ is exactly the \emph{total variation} norm (in the sense of measures). As $\Sch_0(X)$ is dense in $C_0(X)$ \cite[cf.][]{denseness-lizorkin}, we can associate every measure in $\M(X)$ with a Lizorkin distribution and view $\M(X) \subset \Sch_0'(X) \subset \Sch(X)$, providing the description
\[
  \M(X) \coloneqq \curly{u \in \Sch_0'(X)
  \st \norm{u}_{\M(X)} < \infty},
\]
and so the duality pairing $\ang{\dummy, \dummy}$ in \cref{eq:M-norm} can be viewed, formally, as the integral
\[
    \ang{u, \varphi} = \int_{X} \varphi(\vec{x}) u(\vec{x}) \dd \vec{x},
\]
where $u$ is viewed as an element of $\Sch_0'(X)$. In this paper, we will be mostly be working with $X = \cyl$, the Radon domain.

As we will later see in \cref{subsec:radon-transform}, the Radon transform of a function is necessarily even, so we will be interested in the Banach spaces of odd and even finite Radon measures on $\cyl$. Viewing $\M(X) \subset \Sch_0'(X)$, put
\begin{align*}
    \Me(\cyl) &\coloneqq \curly{u \in \M(\cyl) \st u(\vec{\gamma}, t) = u(-\vec{\gamma}, -t)} \\
    \Mo(\cyl) &\coloneqq \curly{u \in \M(\cyl) \st u(\vec{\gamma}, t) = -u(-\vec{\gamma}, -t)}.
\end{align*}
One can then verify that the predual of $\Mo(\cyl)$ is the subspace of odd functions in $C_0(\cyl)$ and the predual of $\Me(\cyl)$ is the subspace of even functions in $C_0(\cyl)$, and so the associated norms of $\Mo(\cyl)$ and $\Me(\cyl)$ can be defined accordingly.

Finally, $\Me(\cyl)$ can equivalently be viewed as $\M(\P^d)$ where $\P^d$ denotes the manifold of hyperplanes in $\R^d$. This follows from the fact that every hyperplane in $\R^d$ takes the form $h_{(\vec{\gamma}, t)} \coloneqq \curly{\vec{x} \in \R^d \st \vec{\gamma}^\T\vec{x} = t}$ for some $(\vec{\gamma}, t) \in \cyl$ and $h_{(\vec{\gamma}, t)} = h_{(-\vec{\gamma}, -t)}$. It will sometimes be convenient to work with $\M(\P^d)$ instead of $\Me(\cyl)$ as $\P^d$ is a locally compact Hausdorff space.

\subsection{The Fourier transform}
The Fourier transform $\FourierOp$ of $f: \R^d \to \mathbb{C}$ and
inverse Fourier transform $\FourierOp^{-1}$ of $F: \R^d \to \mathbb{C}$ are
given by
\begin{align*}
  \Fourier{f}(\vec{\xi}) &\coloneqq \int_{\R^d} f(\vec{x}) e^{-\imag
  \vec{x}^\T\vec{\xi}} \dd \vec{x}, \quad \vec{\xi} \in \R^d \\
  \InvFourier{F}(\vec{x}) &\coloneqq \frac{1}{(2\pi)^d}\int_{\R^d}
  e^{\,\imag \vec{x}^\T \vec{\xi}} F(\vec{\xi}) \dd\vec{\xi}, \quad \vec{x} \in \R^d,
\end{align*}
where $\imag^2 = -1$. We will usually write $\hat{\dummy}$ for $\Fourier{\dummy}$. The Fourier transform and its inverse can be applied to functions in $\Sch(\R^d)$, resulting in functions in $\Sch(\R^d)$. These transforms can be extended to act on $\Sch'(\R^d)$ by duality.

\subsection{The Hilbert transform}
The Hilbert transform $\HilbertOp$ of $f: \R \to \mathbb{C}$ is given by
\[
  \Hilbert{f}(x) \coloneqq \frac{\imag}{\pi} \pv \int_{-\infty}^\infty
  \frac{f(y)}{x - y} \dd y, \quad x \in \R,
\]
where $\pv$ denotes understanding the integral in the Cauchy principal value sense. The prefactor was chosen so that
\[
    \hat{\Hilbert{f}}(\omega) = \paren{\sgn \omega} \hat{f}(\omega)
    \quad\text{and}\quad
    \HilbertOp \HilbertOp f = f.
\]
The Hilbert transform can be applied to functions in $\Sch(\R^d)$, resulting in functions in $\Sch(\R^d)$. This transform can be extended to act on $\Sch'(\R^d)$ by duality.

\subsection{The Radon transform} \label{subsec:radon-transform}
The Radon transform $\RadonOp$ of $f: \R^d \to \R$ and the dual
Radon transform $\RadonOp^*$ of $\Phi: \cyl \to
\R$ are given by
\begin{align*}
  \Radon{f}(\vec{\gamma}, t) &\coloneqq \int_{\curly{\vec{x}: \vec{\gamma}^\T
  \vec{x} = t}} f(\vec{x}) \dd s(\vec{x}), \quad (\vec{\gamma},t) \in \cyl\numberthis \label{eq:radon-transform} \\
  \DualRadon{\Phi}(\vec{x}) &\coloneqq \int_{\Sph^{d-1}}
  \Phi(\vec{\gamma}, \vec{\gamma}^\T \vec{x}) \dd\sigma(\vec{\gamma}), \quad \vec{x} \in \R^d,
\end{align*}
where $s$ denotes the surface measure on the plane $\curly{\vec{x}
\st \vec{\gamma}^\T \vec{x} = t}$, and $\sigma$ denotes the surface
measure on $\Sph^{d-1}$. We will sometimes write $\vec{z} = (\vec{\gamma}, t)$
as the variable in the Radon domain. We discuss the spaces which we can apply the Radon transform and its dual in the sequel. We also remark that the Radon transform of a function is always even, i.e., $\Radon{f}(\vec{\gamma}, t) = \Radon{f}(-\vec{\gamma}, -t)$.

Another way to view the Radon transform and its dual is to consider, formally, the integrals
\begin{align}
    \Radon{f}(\vec{\gamma}, t) &= \int_{\R^d} f(\vec{x}) \delta_\R(\vec{\gamma}^\T\vec{x} - t) \dd \vec{x}, \quad (\vec{\gamma}, t) \in \cyl
    \label{eq:formal-radon-transform} \\
    \DualRadon{\Phi}(\vec{x}) &= \int_{\cyl} \delta_\R(\vec{\gamma}^\T\vec{x} - t) \Phi(\vec{\gamma}, t) \dd(\sigma \times \lambda)(\vec{\gamma}, t), \quad \vec{x} \in \R^d,
    \label{eq:formal-dual-radon-transform}
\end{align}
where $\lambda$ denotes the univariate Lebesgue measure.

The fundamental result of the Radon transform is the \emph{Radon inversion formula}, which states for any $f \in \Sch(\R^d)$
\begin{equation}
  2(2\pi)^{d-1} f = \RadonOp^* \Lambda^{d - 1} \RadonOp f,
  \label{eq:radon-inversion}
\end{equation}
where the \emph{ramp filter} $\Lambda^d$ of a function
$\Phi(\vec{\gamma}, t)$ is given by
\begin{equation}
  \Lambda^d\curly{\Phi}(\vec{\gamma}, t) \coloneqq \begin{cases}
    \partial_t^d \Phi(\vec{\gamma}, t), & \text{$d$ even} \\[0.5ex]
    \HilbertOp_t \partial_t^d \Phi(\vec{\gamma}, t), & \text{$d$ odd},
  \end{cases}
  \label{eq:ramp-filter}
\end{equation}
where $\HilbertOp_t$ is the Hilbert transform (in the variable $t$) and
$\partial_t$ is the partial derivative with respect to $t$. It is easier to see that $\Lambda^d$ is indeed a ramp filter 
by looking at its frequency response with respect to the $t$ variable. We have
\[
  \hat{\Lambda^d \Phi}(\vec{\gamma}, \omega) = \imag^d \abs{\omega}^d
  \hat{\Phi}(\vec{\gamma}, \omega).
\]

Some care has to be taken to understand the Radon transforms of distributions. Just as the Fourier and Hilbert transforms can be extended to distributions via duality, we do the same with the Radon transform, though some care has to be taken. In particular, the choice of test functions must be carefully chosen and cannot be the space of Schwartz functions. It is easy to verify that if $\varphi \in \Sch(\R^d)$, then $\Radon{\varphi} \in \Sch(\cyl)$. This is not true about the dual transform. Indeed, if $\psi \in \Sch(\cyl)$, then it may not be true that $\DualRadon{\psi} \in \Sch(\R^d)$.

Due to recent developments in ridgelet analysis~\citep{ridgelet-transform-distributions}, specifically regarding the continuity of the Radon transform of Lizorkin test functions, we have the following result.
\begin{proposition}[{\citet[][Corollary~6.1]{ridgelet-transform-distributions}}]
  \label[Proposition]{thm:radon-bijections}
  The transforms
  \begin{align*}
    \RadonOp: \Sch_0(\R^d) \to \Sch_0(\P^d) \\
    \RadonOp^*: \Sch_0(\P^d) \to
    \Sch_0(\R^d)
  \end{align*}
  are continuous bijections, where $\Sch_0(\P^d) \subset \Sch_0(\cyl)$ denotes the subspace of even Lizorkin test functions.
\end{proposition}

\Cref{thm:radon-bijections} allows us to define the Radon transform and dual Radon transform of distributions by duality by choosing our test functions to be Lizorkin test functions, i.e., the action of the Radon transform of $f \in \Sch_0'(\R^d)$ on $\psi \in \Sch_0(\P^d)$ is defined to be
$\ang{\RadonOp f, \psi} \coloneqq \ang{f, \RadonOp^* \psi}$,
and the action of the dual Radon transform of $\Phi \in \Sch_0'(\P^d)$ on $\varphi \in \Sch_0(\R^d)$ is defined to be
$\ang{\RadonOp^* \Phi, \varphi} \coloneqq \ang{\Phi, \RadonOp \varphi}$. This means we have the following corollary to \cref{thm:radon-bijections}.
\begin{corollary} \label[Corollary]{cor:radon-bijections}
  The transforms
  \begin{align*}
    \RadonOp: \Sch_0'(\R^d) \to \Sch_0'(\P^d) \\
    \RadonOp^*: \Sch_0'(\P^d) \to \Sch_0'(\R^d)
  \end{align*}
  are continuous bijections.
\end{corollary}

We also have the following inversion formula for the
dual Radon transform~\cite[Theorem~3.7]{integral-geometry-radon-transforms}. For
any $\Phi \in \Sch_0(\P^d)$
\begin{equation}
  2(2\pi)^{d-1} \Phi = \Lambda^{d-1} \RadonOp \RadonOp^* \Phi.
  \label{eq:dual-radon-inversion}
\end{equation}

The inversion formulas in \cref{eq:radon-inversion,eq:dual-radon-inversion} can
be rewritten in many ways using the \emph{intertwining relations} of the Radon
transform and its dual with the Laplacian operator~\cite[Lemma~2.1]{integral-geometry-radon-transforms}. We have
\begin{equation}
(-\Delta)^{\frac{d-1}{2}} \RadonOp^* = \RadonOp^* \Lambda^{d
- 1} \qquad\text{and}\qquad \RadonOp (-\Delta)^{\frac{d-1}{2}} = \Lambda^{d -
1} \RadonOp.
\label{eq:radon-intertwining}
\end{equation}
As the constant $2(2\pi)^{d-1}$ arises often when working with the Radon transform, we put
\begin{equation}
    c_d \coloneqq \frac{1}{2(2\pi)^{d-1}}.
    \label{eq:cd}
\end{equation}

\begin{remark}
    We warn the reader that here and in the rest of the paper, we use the pairing $\ang{\dummy, \dummy}$ to generically denote the duality pairing between a space and its continuous dual. We will not use different notation for different pairings to reduce clutter. The exact pairings should be clear from context.
\end{remark}

With these definitions in hand, we see that the seminorms \cref{eq:seminorms} studied in this paper are well-defined. Recall from \cref{eq:seminorms} that
\[
\norm{f}_{(m)} = \norm{\ROp_m f}_{\M(\cyl)} = c_d \norm{\partial_t^m \ramp^{d-1} \RadonOp f}_{\M(\cyl)},
\]
where $f \in \F_m$. We will later show in \cref{lemma:finite-dim-null-space} that the null space $\N_m$ of $\ROp_m$ is the space of polynomials of degree strictly less than $m$. Thus, to understand that the seminorms studied in this paper are well-defined, we can view $f \in \Sch_0'(\R)$. From \cref{cor:radon-bijections}, it follows that $\RadonOp f \in \Sch_0'(\P^d) \subset \Sch_0'(\cyl) \subset \Sch'(\cyl)$. Since $\partial_t^m \ramp^{d-1}$ is a Fourier multiplier and from the definition of $\F_m$, we see that $\partial_t^m \ramp^{d-1} \RadonOp f \in \M(\cyl)$, and so the seminorms are well-defined.
 \section{The Representer Theorem} \label{sec:rep-thm}
In this section we will prove \cref{thm:rep-thm}, our representer theorem. The general strategy will be to reduce the problem in \cref{eq:inverse-problem} to one that is similar to the classical problem of \emph{Radon measure recovery}, which has been studied since as early as the 1930s~\citep{radon-measure-recovery1,radon-measure-recovery2}. Let $\Omega \subset \R^d$ be a bounded domain. The prototypical Radon measure recovery problem studies optimizations of the form
\begin{equation}
    \min_{u \in \M(\Omega)} \: \norm{u}_{\M(\Omega)} \quad\subj\quad \sensing u = \vec{y},
    \label{eq:radon-measure-recovery}
\end{equation}
where $\sensing: \M(\Omega) \to \R^N$ is a continuous linear operator and $\vec{y} \in \R^N$. The first ``representer theorem'' for \cref{eq:radon-measure-recovery} is from~\cite{rep-thm-radon-measure-recovery}. This representer theorem essentially states that there exists a sparse solution to \cref{eq:radon-measure-recovery} of the form
\[
    \sum_{k = 1}^K v_k \, \delta_{\R^d}(\dummy - \vec{x}_k),
\]
with $K \leq N$. Refinements to this result have been made over the years, e.g.,~\cite[Theorem~1]{fisher-jerome}, including very modern results, e.g.,~\cite[Theorem~7]{L-splines},~\cite[Section~4.2.3]{rep-thm-convex-reg}, and~\cite[Theorem~4.2]{sparsity-variational-inverse}. For our problem, we reduce \cref{eq:inverse-problem} to one of the following Radon measure recovery problems in \cref{prop:radon-measure-recovery-even,prop:radon-measure-recovery-odd}, depending on if $m$ is even or odd. Reducing our problem to one of these Radon measure recovery problems requires several steps. The first step is to understand what functions are sparsified by $\ROp_m$. The next step is the construction of a stable right-inverse of $\ROp_m$. The final step is to understand the topological structure of the native space $\F_m$. These are outlined in the remainder of this section before finally proving \cref{thm:rep-thm} at the end of this section.

\begin{proposition}[{Based on~\citet[Theorem~4.2]{sparsity-variational-inverse}}] \label[Proposition]{prop:radon-measure-recovery-even}
    Assume the following: 
    \begin{enumerate}[label=\arabic*.]
        \item The function $\datafit: \R^N \to \R$ is strictly convex, coercive, and lower semi-continuous.
        \item The operator $\sensing: \M(\P^d) \to \R^N$ is continuous, linear, and surjective, where we recall that $\P^d$ is the manifold of hyperplanes in $\R^d$.
    \end{enumerate}
    Then, there exists a sparse minimizer to the Radon measure recovery problem
    \[
        \min_{u \in \M(\P^d)} \: \datafit(\sensing u) + \norm{u}_{\M(\P^d)}
    \]
    of the form
    \[
        \bar{u}
        = \sum_{k=1}^K v_k\, \delta_{\P^d}(\dummy - \vec{z}_k)
        = \sum_{k=1}^K v_k\, \sq{\frac{\delta_\cyl(\dummy - \vec{z}_k) + \delta_\cyl(\dummy + \vec{z}_k)}{2}}
    \]
    with $K \leq N$, $v_k \in \R \setminus \curly{0}$, and $\vec{z}_k = (\vec{w}_k, b_k) \in \cyl$, $k = 1, \ldots, K$.
\end{proposition}
Although \citet[Theorem~4.2]{sparsity-variational-inverse} considers an open, bounded subset of $\R^d$ rather than $\P^d$, their proof is general enough to apply to any locally compact Hausdorff space, and, in particular, for $\P^d$. Analogously, we immediately have the following result for the Radon measure recovery problem posed over $\Mo(\cyl)$.

\begin{proposition} \label[Proposition]{prop:radon-measure-recovery-odd}
    Assume the following: 
    \begin{enumerate}[label=\arabic*.]
        \item The function $\datafit: \R^N \to \R$ is a strictly convex function that is coercive and lower semi-continuous.
        \item The operator $\sensing: \Mo(\cyl) \to \R^N$ is continuous, linear, and surjective.
    \end{enumerate}
    Then, there exists a sparse minimizer to the Radon measure recovery problem
    \[
        \min_{u \in \Mo(\cyl)} \: \datafit(\sensing u) + \norm{u}_{\Mo(\cyl)}
    \]
    of the form
    \[
        \bar{u}
        = \sum_{k=1}^K v_k\, \sq{\frac{\delta_\cyl(\dummy - \vec{z}_k) - \delta_\cyl(\dummy + \vec{z}_k)}{2}}
    \]
    with $K \leq N$, $v_k \in \R \setminus \curly{0}$, and $\vec{z}_k = (\vec{w}_k, b_k) \in \cyl$, $k = 1, \ldots, K$.
\end{proposition}

In order to reduce \cref{eq:inverse-problem} to one of the problems in \cref{prop:radon-measure-recovery-even,prop:radon-measure-recovery-odd}, we need to understand which functions are sparsified by $\ROp_m$. As claimed in \cref{rem:nn-atoms-sparsified}, these are exactly the neurons in \cref{eq:ridge-spline}.
\begin{lemma} \label[Lemma]{lemma:nn-atoms-sparsified}
    The atoms of the single-hidden layer neural network as in \cref{eq:ridge-function} are ``sparsified'' by $\ROp_m$ in the sense that
    \[
        \ROp_m r_{(\vec{w}, b)}^{(m)} = \frac{\delta_\cyl(\dummy - \vec{z}) + (-1)^m \delta_\cyl(\dummy + \vec{z})}{2},
    \]
    where $\vec{z} = (\vec{w}, b) \in \cyl$ and $r_{(\vec{w}, b)}^{(m)}(\vec{x}) \coloneqq \rho_m(\vec{w}^\T\vec{x} - b)$.
\end{lemma}
Before proving \cref{lemma:nn-atoms-sparsified}, we first prove the following intermediary result which may be of independent interest.
\begin{lemma} \label[Lemma]{lemma:radon-integrator}
    Let $r_{(\vec{\gamma}, t)}^{(m)}(\vec{x}) \coloneqq \rho_m(\vec{\gamma}^\T\vec{x} - t)$. For all $\varphi \in \Sch_0(\R^d)$ and any $m \in \mathbb{N}$ we have
    \[
        \ang{r_{(\vec{\gamma}, t)}^{(m)}, \varphi} = (-1)^m(\partial_t^{-m}) \Radon{\varphi}(\vec{\gamma}, t),
    \]
    where $\partial_t^{-m}$ is a stable left-inverse of $\partial_t^m$, i.e., an $m$-fold integrator in the $t$ variable of the Radon domain. In particular, by the intertwining relations of the Radon transform with the Laplacian operator in \cref{eq:radon-intertwining}, this says for even $m$
    \[
        \ang{\Delta^{m/2}r_{(\vec{\gamma}, t)}^{(m)}, \varphi} = \Radon{\varphi}(\vec{\gamma}, t),
    \]
    for all $\varphi \in \Sch_0(\R^d)$.
\end{lemma}
\begin{proof}
    The proof is a direct computation. For all $\varphi \in \Sch_0(\R^d)$ and any $m \in \mathbb{N}$ we have
    \begin{align*}
        \ang{r_{(\vec{\gamma}, t)}^{(m)}, \varphi}
        &= \int_{\R^d} \rho_m(\vec{\gamma}^\T\vec{x} - t) \varphi(\vec{x}) \dd\vec{x} \\
        &= \int_{\R^d} (-1)^m(\partial_t^{-m})\delta_\R(\vec{\gamma}^\T\vec{x} - t) \varphi(\vec{x}) \dd\vec{x} \\
        &= (-1)^m(\partial_t^{-m}) \int_{\R^d} \delta_\R(\vec{\gamma}^\T\vec{x} - t) \varphi(\vec{x}) \dd\vec{x} \\
        &= (-1)^m(\partial_t^{-m}) \Radon{\varphi}(\vec{\gamma}, t).
    \end{align*}
    In particular, the stability of $\partial_t^{-m}$ ensures that the third line is well-defined\footnote{One choice of $\partial_t^{-m}$ appears in \citet[pg. 781]{L-splines}, where the (Schwartz) kernel of $\partial_t^{-m}$ is compactly supported and bounded, which justifies changing the order of the operator and the integral.}. The last line follows from \cref{eq:formal-radon-transform}.
\end{proof}
\begin{proof}[Proof of \cref{lemma:nn-atoms-sparsified}]
    One may verify that when $m$ is even, $\ROp_m f$ is even and when $m$ is odd, $\ROp_m f$ is odd. Next, as discussed in \cref{rem:radon-impulse}, if we consider the subspace of even or odd Lizorkin distributions, then the point evaluation functionals (i.e., Dirac impulses) are
    \[
        \frac{\delta_\cyl(\dummy - \vec{z}_k) + \delta_\cyl(\dummy + \vec{z}_k)}{2}
    \]
    for the even subspace and
    \[
        \frac{\delta_\cyl(\dummy - \vec{z}_k) - \delta_\cyl(\dummy + \vec{z}_k)}{2}
    \]
    for the odd subspace. Thus, it suffices to check that
    \begin{equation}
        c_d\ang{\partial_t^m \ramp^{d-1} \RadonOp r_{(\vec{w}, b)}^{(m)}, \psi} = \psi(\vec{w}, b)
        \label{eq:radon-measure-recovery-verify}
    \end{equation}
    for either all even Lizorkin test functions $\psi \in \Sch_0(\cyl)$ if $m$ is even or for all odd Lizorkin test functions $\psi \in \Sch_0(\cyl)$ is $m$ is odd. We remark that for the even (respectively odd) subspace, the pairing $\ang{\dummy, \dummy}$ in the above display is of even (respectively odd) Lizorkin test functions and even (respectively odd) Lizorkin distributions.
    
    We now verify \cref{eq:radon-measure-recovery-verify}. Suppose that $m$ is even. For all even $\psi \in \Sch_0(\cyl)$ we have
    \begin{align*}
        c_d\ang{\partial_t^m \ramp^{d-1} \RadonOp r_{(\vec{w}, b)}^{(m)}, \psi}
        &= c_d \ang{r_{(\vec{w}, b)}^{(m)}, \RadonOp^*\ramp^{d-1}\curly{(-1)^m \partial_t^m \psi}} \\
        &= \sq{(-1)^m \partial_t^{-m} c_d \RadonOp \RadonOp^*\ramp^{d-1}\curly{(-1)^m \partial_t^m \psi}}(\vec{w}, b) \\
        &= \sq{(-1)^m \partial_t^{-m} (-1)^m \partial_t^m \psi}(\vec{w}, b) \\
        &= \psi(\vec{w}, b),
    \end{align*}
    where the first line holds since the formal adjoint of $\partial_t^m$ is $(-1)^m \partial_t^m$, the second line holds via \cref{lemma:radon-integrator}, the third line holds by the dual Radon transform inversion formula \cref{eq:dual-radon-inversion} and the intertwining relations \cref{eq:radon-intertwining} combined with the fact that since $\psi$ is even and $\psi \in \Sch_0(\cyl)$ we have that $\partial_t^m \psi$ is also even and $\partial_t^m \psi \in \Sch_0(\cyl)$, and the fourth line holds by the left-inverse property for our choice of construction for $\partial_t^{-m}$.
The case when $m$ is odd is analogous, with the key fact that if $\psi$ is odd and $\psi \in \Sch_0(\cyl)$ we have that $\partial_t^m \psi$ is \emph{even} and $\partial_t^m \psi \in \Sch_0(\cyl)$, which justifies the use of the dual Radon transform inversion formula.
\end{proof}

Since $\ROp_m$ sparsifies functions of the form $r_{(\vec{w}, b)}^{(m)} = \rho_m(\vec{w}^\T\vec{x} - b)$, we choose to define the growth restriction previously discussed in \cref{subsec:rep-thm} for the null space and native space of $\ROp_m$ via the algebraic growth rate
\begin{equation}
    n_0 \coloneqq \inf \curly{n \in \mathbb{N} \st r_{(\vec{w}, b)}^{(m)} \in L^{\infty, n}(\R^d)} = m - 1.
\label{eq:growth-restriction}
\end{equation}

In \cref{eq:inverse-problem}, we are optimizing over the native space $\native_m$. Although $\native_m$ is defined by the seminorm $\norm{\ROp_m f}_{\M(\cyl)}$, we show in \cref{thm:banach-space} that, if we equip $\F_m$ with the proper direct-sum topology, it forms a bona fide Banach space. In order to prove that $\F_m$ is a Banach space, we require two intermediary results.
\begin{lemma} \label[Lemma]{lemma:finite-dim-null-space}
The null space $\N_m$ of $\ROp_m$ defined in \cref{eq:null-space} is finite-dimensional. In particular, it is the space of polynomials of degree strictly less than $m$.
\end{lemma}
\Cref{lemma:finite-dim-null-space} says, in particular, that we can find a \emph{biorthogonal system} for $\N_m$.
\begin{definition} \label[Definition]{defn:biorthogonal-system}
    Consider a finite-dimensional space $\N$ with $N_0 \coloneqq \dim \N$. The pair $(\vec{\phi}, \vec{p}) = \curly{(\phi_n, p_n)}_{n=1}^{N_0}$ is called a \emph{biorthogonal system} for $\mathcal{N}$ if $\curly{p_n}_{n=1}^{N_0}$ is a basis of $\N$ and the vector of ``boundary'' functionals $\vec{\phi} = (\phi_1, \ldots, \phi_{N_0})$ with $\phi_n \in \N'$ (the continuous dual of $\N$) satisfy the biorthogonality condition $\ang{\phi_k, p_n} = \delta[k - n]$, $k, n = 1, \ldots, N_0$, where $\delta[\dummy]$ is the Kronecker impulse.
\end{definition}

Put $N_0 \coloneqq \dim \N_m$ and let $(\vec{\phi}, \vec{p})$ be a biorthogonal system for $\N_m$. \Cref{defn:biorthogonal-system} says that any $q \in \N_m$ has the \emph{unique representation}
\[
  q = \sum_{n = 1}^{N_0} \ang{\phi_n, q} p_n.
\]
We will sometimes write $\vec{\phi}(f)$ to denote the vector $(\ang{\phi_1, f}, \ldots, \ang{\phi_{N_0}, f}) \in \R^{N_0}$.
\begin{lemma}
\label[Lemma]{lemma:right-inverse}
    Let $(\vec{\phi}, \vec{p})$ be a biorthogonal system for $\N_m \subset \F_m \subset L^{\infty, m-1}(\mathbb{R}^d)$. Then, there exists a \emph{unique operator}
    \[
    \ROp_{m, \vec{\phi}}^{-1}: \psi \mapsto \ROp_{m,
    \vec{\phi}}^{-1}\psi = \int_{\mathbb{S}^{d-1} \times \mathbb{R}}
    g_{m, \vec{\phi}}(\cdot, \vec{z}) \psi(\vec{z}) \dd(\sigma \times \lambda)(\vec{z}),
    \]
    where we recall that $\sigma$ is the surface measure on $\Sph^{d-1}$ and $\lambda$ is the univariate Lebesgue measure. Then, for all even $\psi \in \Sch_0(\cyl)$ if $m$ is even and all odd $\psi \in \Sch_0(\cyl)$ if $m$ is odd, the operator $\ROp_{m, \vec{\phi}}^{-1}$ satisfies
    \begin{equation}
    \begin{aligned}
      \ROp_m\ROp_{m, \vec{\phi}}^{-1}\psi &= \psi \qquad\text{(right-inverse
      property)} \\
      \vec{\phi}(\ROp_{m, \vec{\phi}}^{-1}\psi) &= \vec{0}
      \qquad\text{(boundary conditions)}
    \end{aligned}
    \label{eq:stable-right-inverse-props}
    \end{equation}
    The kernel of this operator is
    \[
    g_{m, \vec{\phi}}(\vec{x}, \vec{z}) = r_\vec{z}^{(m)}(\vec{x}) -
    \sum_{n=1}^{N_0} p_n(\vec{x}) q_n(\vec{z}),
    \]
    where $r_\vec{z}^{(m)} = r_{(\vec{w}, b)}^{(m)} = \rho_m(\vec{w}^\T(\dummy) - b)$ and $q_n(\vec{z}) \coloneqq \ang{\phi_n, r_\vec{z}}$.
    
    If $\ROp_{m, \vec{\phi}}^{-1}$ is bounded, then it admits a continuous extension from $\Me(\cyl)$ or $\Mo(\cyl)$ to $L^{\infty, m-1}(\mathbb{R}^d)$ when $m$ is even or odd, with \cref{eq:stable-right-inverse-props} holding for all $\psi \in \Me(\cyl)$ or $\psi \in \Mo(\cyl)$.
\end{lemma}

With these two results, we can now establish the Banach space structure of $\native_m$.
\begin{theorem}
\label{thm:banach-space}
Let $(\vec{\phi}, \vec{p})$ be a biorthogonal system for the null space $\N_m$ of $\ROp_m$ as defined in \cref{eq:null-space} and let $\native_m$ be the native space of $\ROp_m$ as defined in \cref{eq:native-space}. Then, the following hold:
\begin{enumerate}[label=\arabic*.]
    \item  The right-inverse operator $\ROp_{m, \vec{\phi}}^{-1}$ specified by \cref{lemma:right-inverse} isometrically maps $\Me(\cyl)$ (respectively $\Mo(\cyl)$) to $\native_m$ when $m$ is even (respectively odd). Moreover, this map is necessarily bounded.

    \item Every $f \in \F_m$ admits a \emph{unique} representation
    \begin{equation}
        f = \ROp_{m, \vec{\phi}}^{-1} u + q,
        \label{eq:unique-representation}
    \end{equation}
    where $u = \ROp_m f \in \mathcal{M}(\cyl)$\footnote{More specifically, we have that $u \in \Me(\cyl)$ when $m$ is even and $u \in \Mo(\cyl)$ when $m$ is odd.} and $q = \sum_{n=1}^{N_0} \ang{\phi_n, f}p_n \in \N_m$. In particular, this specifies the structural property $\F_m = \F_{m,\vec{\phi}} \oplus
    \N_m$, where
    \[
        \F_{m,\vec{\phi}} \coloneqq \curly{f \in
        \F \st \vec{\phi}(f) = \vec{0}}.
    \]
    \label{item:unique-rep}
    
    \item $\F_m$ is a Banach space when equipped with the norm
    \[
        \norm{f}_{\F_m} \coloneqq \norm{\ROp_m f}_{\M(\cyl)} + \norm{\vec{\phi}(f)}_{2}.
    \]
    \label{item:banach-norm}
\end{enumerate}
\end{theorem}

\begin{remark}
     \cref{item:unique-rep} in \cref{thm:banach-space} says that every $f \in \F_m$ admits an \emph{integral representation} via \cref{eq:unique-representation}, that can be viewed as an infinite-width (continuum-width) neural network. Since $u$ in \cref{eq:unique-representation} depends on $f$, it follows that \cref{eq:unique-representation} is a kind of \emph{Calder{\'o}n-type reproducing formula}~\citep{calderon-reproducing}. Integral representations have been studied in several recent works~\citep{convex-nn, function-space-relu,dimension-independent-approx-bounds}. We also remark that \cref{eq:unique-representation} shares many similarities with the \emph{dual Ridgelet transform}~\citep{murata-ridgelets,rubin-ridgelets,candes-phd,candes-ridgelets}.
\end{remark}

The proofs of \cref{lemma:finite-dim-null-space,lemma:right-inverse,thm:banach-space} appear in \cref{app:aux-proofs}. We will now prove \cref{thm:rep-thm}.

\subsection{Proof of \cref{thm:rep-thm}}
\begin{proof}
We first recast the problem in \cref{eq:inverse-problem} as one with interpolation constraints. To do this use a technique from~\cite{unifying-representer}. We use the fact that $G$ is a strictly convex function. In particular, for any two solutions $\bar{f}$, $\tilde{f}$ of \cref{eq:inverse-problem}, we must have $\sensing \bar{f} = \sensing \tilde{f}$ (since otherwise, it would contradict the strict convexity of $G$). Hence, there exists $\vec{z} \in \R^N$ such that $\vec{z} = \sensing \bar{f} = \sensing \tilde{f}$. Although $\vec{z} \in \R^N$ is not usually known before hand, this property provides us with the parametric characterization of the solution set to \cref{eq:inverse-problem} as
\begin{equation}
    S_\vec{z} \coloneqq \argmin_{f \in \F_m} \: \norm{\ROp_m f}_{\M(\cyl)} \quad\subj\quad \sensing f = \vec{z},
    \label{eq:interpolation-constraints}
\end{equation}
for some $\vec{z}\in\R^N$. Hence, it suffices to show that there exists a solution to \cref{eq:interpolation-constraints} of the form in \cref{eq:ridge-spline}. We will now show this for a fixed $\vec{z} \in \R^N$.

Consider the $N \times N_0$ matrix
\begin{equation}
    \mat{A} \coloneqq \begin{bmatrix}
        \sensing p_1 & \cdots & \sensing p_{N_0}
    \end{bmatrix},
    \label{eq:thm-1-system}
\end{equation}
where $\curly{p_n}_{n=1}^{N_0}$ is a basis for $\N_m$. Since every $q \in \N_m$ has a \emph{unique} expansion $q = \sum_{n=1}^{N_0} c_n p_n$, we have that the linear system $\mat{A} \vec{c} = \sensing q$, where $\vec{c} = (c_1, \ldots, c_{N_0})$, has a unique solution. From \cref{item:well-posed-null-space} in \cref{thm:rep-thm}, we see that the system in \cref{eq:thm-1-system} satisfies $N \geq N_0$. Hence, it is solvable if and only if $\mat{A}^\T\mat{A}$ is invertible and the solution is given by the least-squares solution
\[
    \vec{c} = (\mat{A}^\T \mat{A})^{-1} \mat{A}^\T (\sensing q).
\]

Thus, we know $\mat{A}^\T \mat{A}$ must be invertible. Invertibility of $\mat{A}^\T \mat{A}$ says, in particular, that $\spn\curly{\vec{a}_n}_{n=1}^N = \R^{N_0}$, where $\vec{a}_n^\T$ is the $n$th row of $\mat{A}$. Therefore, there exists a subset of $N_0$ rows of $\mat{A}$ that span $\R^{N_0}$. Without loss of generality, suppose this subset is $\curly{\vec{a}_n}_{n=1}^{N_0}$. Then, the submatrix $\mat{A}_0$ of $\mat{A}$ defined by
\[
    \mat{A}_0 \coloneqq \begin{bmatrix}
        \vec{a}_1^\T \\ \vdots \\ \vec{a}_{N_0}^\T
    \end{bmatrix}
\]
is invertible. Consider the components $(\nu_1, \ldots, \nu_N)$ of $\sensing$ via $\sensing: f \mapsto (\ang{\nu_1, f}, \ldots, \ang{\nu_N, f})$. We can write $\vec{a}_n$ as the vector $(\ang{\nu_n, p_1}, \ldots, \ang{\nu_n, p_{N_0}})$, $n = 1, \ldots, N_0$. Hence the reduced subset of measurements $(\nu_1, \ldots, \nu_{N_0})$ are linearly independent with respect to $\N_m$. Let $\sensing_0$ denote this reduced set of measurements and let $\sensing_1$ denote the remaining set of measurements, i.e., $\sensing = (\sensing_0, \sensing_1)$.

Next, notice that $\sensing_0 p_n$ is the $n$th column of $\mat{A}_0$. Let $\vec{e}_n$ denote the $n$th canonical basis vector. Then, we have the equality $\sensing_0 p_n = \mat{A}_0 \vec{e}_n$. Using the invertibility of $\mat{A}_0$, we have
\[
    \mat{A}_0^{-1} (\sensing_0 p_n) = \vec{e}_n.
\]
If we put $\vec{\phi}_0 \coloneqq \mat{A}_0^{-1} \circ \sensing_0$, the above display is exactly the biorthogonality property and hence $(\vec{\phi}_0, \vec{p})$ form a biorthogonal system for $\N_m$.

One can verify that
\begin{equation}
    \norm{\ROp_m f}_{\M(\cyl)} = \begin{cases}
        \norm{\ROp_m f}_{\Me(\cyl)}, & \text{$m$ even} \\
        \norm{\ROp_m f}_{\Mo(\cyl)}, & \text{$m$ odd}.
    \end{cases}
    \label{eq:odd-even-norm}
\end{equation}

Suppose that $m$ is even. By \cref{thm:banach-space}, every $f \in \F_m$ admits a unique representation $f = \ROp_{m, \vec{\phi}_0}^{-1} u + q_0$, where $u  = \ROp_m f \in \Me(\cyl)$ and $q_0 \in \N_m$. We can use the above display to rewrite the problem in \cref{eq:interpolation-constraints} as
\begin{equation}
    \begin{aligned}
      \min_{u \in \Me(\cyl)} \quad & \norm{u}_{\Me(\cyl)}  \\
      \subj \quad & \sensing f = \vec{z}, \\
      & f = \ROp_{m, \vec{\phi}}^{-1} u + q_0.
    \end{aligned}
    \label{eq:problem-over-radon-measures}
\end{equation}
Write $\vec{z}_0 = (z_1, \ldots, z_{N_0})$ and $\vec{z}_1 = (z_{N_0 + 1}, \ldots, z_N)$. Then, the constraint $\sensing f = \vec{z}$ can be written as the two constraints $\sensing_0 f = \vec{z}_0$ and $\sensing_1 f = \vec{z}_1$. By the boundary conditions in \cref{eq:stable-right-inverse-props} we have that $\vec{\phi}_0(f) = \vec{\phi}_0(\ROp_{m, \vec{\phi}_0}^{-1} u + q_0) = \vec{\phi}_0(\ROp_{m, \vec{\phi}_0}^{-1} u) + \vec{\phi}_0(q_0) = \vec{\phi}_0(q_0)$. Thus, by definition of $\vec{\phi}_0$, we have that $\vec{z}_0 = \sensing_0 f = \sensing_0 q_0$. Hence, \cref{eq:problem-over-radon-measures} can be rewritten as
\[
    \begin{aligned}
      \min_{u \in \Me(\cyl)} \quad & \norm{u}_{\Me(\cyl)}  \\
      \subj \quad & \sensing_1 f = \vec{z}_1, \\
      & f = \ROp_{m, \vec{\phi}}^{-1} u + q_0
    \end{aligned}
    \quad=\quad
    \begin{aligned}
      \min_{u \in \Me(\cyl)} \quad & \norm{u}_{\Me(\cyl)}  \\
      \subj \quad & \sensing_1 (\ROp_{m, \vec{\phi}_0}^{-1} u) = \vec{z}_1 - \sensing_1 q_0, \\
      & f = \ROp_{m, \vec{\phi}}^{-1} u + q_0.
    \end{aligned}
\]
This says there exists a solution to the above display of the form $\bar{f} = \ROp_{m, \vec{\phi}_0}^{-1} \bar{u} + q_0$, where
\[
    \bar{u} \in \argmin_{u \in \Me(\cyl)} \norm{u}_{\Me(\cyl)} \quad\subj\quad \sensing_1 (\ROp_{m, \vec{\phi}_0}^{-1} u) = \vec{z}_1 - \sensing_1 q_0.
\]
By \cref{prop:radon-measure-recovery-even}, there exists a sparse minimizer to the above display with $N - N_0$ terms. The result then follows by invoking \cref{lemma:right-inverse} and \cref{lemma:nn-atoms-sparsified}, which says $\bar{f} = \ROp_{m, \vec{\phi}_0}^{-1} \bar{u} + q_0$ takes the form in \cref{eq:ridge-spline}. An analogous argument holds when $m$ is odd by invoking \cref{prop:radon-measure-recovery-odd} instead of \cref{prop:radon-measure-recovery-even}.
\end{proof}

 \section{Ridge Splines and Polynomial Splines} \label{sec:splines}
In this section we establish connections between ridge splines and classical polynomial splines in both the univariate ($d = 1$) and multivariate ($d > 1$) cases.

\subsection{Univariate ridge splines are univariate polynomial splines}
\label{subsec:1D-splines}
Univariate ridge splines and classical univariate splines are in fact the same object. To see this, it suffices to verify that
\[
    \norm{\ROp_m f}_{\M(\cyl)} = c_d \norm{\partial_t^m \ramp^{d-1} \RadonOp f}_{\M(\cyl)} = \norm{\D^mf}_{\M(\R)},
\]
when $d = 1$ and then simply invoke the result of \citet{L-splines}, where $\D^m$ is the univariate $m$th derivative operator. Certainly this is true. Indeed, when $d = 1$, we have that $c_d = 1/2$ and from \cref{eq:formal-radon-transform} that the univariate Radon transform is simply
\[
    \Radon{f}(\gamma, t)
    = \int_\R f(x) \delta_\R(\gamma x - t) \dd x
    = \int_\R \frac{f(x)}{\abs{\gamma}} \delta_\R\paren{x - \frac{t}{\gamma}} \dd x
    = \int_\R f(x) \delta_\R\paren{x - \frac{t}{\gamma}} \dd x
    = f\paren{\frac{t}{\gamma}},
\]
where the second equality holds since the Dirac impulse is homogeneous of degree $-1$ and the third equality holds since $\gamma \in \Sph^0 = \curly{-1, +1}$. Thus,
\begin{align*}
    \eval{c_d\norm{\partial_t^m \ramp^{d-1} \RadonOp f}_{\M(\cyl)}}_{d=1}
    &= \frac{1}{2}\norm{\partial_t^m \RadonOp f}_{\M(\curly{-1, +1} \times \R)} \\
    &= \frac{1}{2}\sum_{\gamma \in \curly{-1, +1}} \norm{\D^m f\paren{\frac{\dummy}{\gamma}}}_{\M(\R)} \\
    &= \norm{\D^mf}_{\M(\R)},
\end{align*}
where the last equality holds since $f(\dummy / \gamma)$ is either $f$ or its reflection, both of which will have the same $\norm{\D^m\curly{\dummy}}_{\M(\R)}$ value. Thus, the main result from the framework of $\Ell$-splines~\citep{L-splines}, we see that univariate polynomial ridge splines of order $m$ are exactly the same as classical univariate polynomial splines of order $m$. This connection between regularized univariate single-hidden layer neural networks and classical notions of univariate splines being fit to data have been recently explored in~\cite{relu-linear-spline,min-norm-nn-splines}. This says, by \cref{prop:equiv-opts}, that training a wide enough univariate neural network with either an appropriate path-norm regularizer or an appropriate weight decay regularizer on data results in an optimal polynomial spline fit of the data. Moreover, these splines are in fact the well-known \emph{locally adaptive regression splines} of~\cite{locally-adaptive-regression-splines}.

\subsection{Ridge splines correspond to univariate splines in the Radon domain} \label{subsec:spline-radon-domain}
Another way to view a ridge spline is as a continuum of univariate polynomial splines in the Radon domain, where the continuum is indexed by directions $\vec{\gamma} \in \Sph^{d-1}$. Suppose $\sensing$ corresponds to the ideal sampling setting where the sampling locations are located at $\curly{\vec{x}_n}_{n=1}^N \subset \R^d$. Then, using the same technique we did in the proof of \cref{thm:rep-thm}, we can recast the continuous-domain inverse problem in \cref{eq:inverse-problem} as one with interpolation constraints:
\begin{equation}
    \min_{f \in \F_m} \: \norm{\ROp_m f}_{\M(\cyl)} \quad\subj\quad f(\vec{x}_n) = z_n, \: n = 1, \ldots, N,
    \label{eq:inverse-problem-interp-constraints}
\end{equation}
for some $\vec{z} \in \R^N$. By \cref{eq:radon-inversion}, the Radon inversion formula, we can always write $f = c_d\RadonOp^* \Lambda^{d - 1} \RadonOp f$ for any $f \in \F_m$, where the operators are understood in the distributional sense via \cref{cor:radon-bijections}. Thus, we see that the above optimization can be rewritten as
\[
    \min_{f \in \F_m} \: c_d\norm{\partial_t^m \ramp^{d-1} \RadonOp f}_{\M(\cyl)} \quad\subj\quad (c_d\RadonOp^* \Lambda^{d - 1} \RadonOp f)(\vec{x}_n)  = z_n, \: n = 1, \ldots, N.
\]
If we put $\Phi \coloneqq c_d \Lambda^{d-1}\RadonOp f$, then the above optimization is
\begin{equation}
    \min_{\Phi \in \mathfrak{F}_m} \: \norm{\partial_t^m \Phi}_{\M(\cyl)} \quad\subj\quad \DualRadon{\Phi}(\vec{x}_n) = \int_{\Sph^{d-1}} \Phi(\vec{\gamma}, \vec{\gamma}^\T\vec{x}_n) \dd \sigma(\vec{\gamma}) = z_n, \: n = 1, \ldots, N,
    \label{eq:radon-domain-opt}
\end{equation}
where $\mathfrak{F}_m$ is the image of $c_d \ramp^{d-1} \RadonOp$ applied to $\F_m$. This essentially says for a fixed direction $\vec{\gamma} \in \Sph^{d-1}$, the function $\bar{\Phi}(\vec{\gamma}, \dummy): \R \to \R$ is an $m$th-order polynomial spline. This follows by considering an optimization for each $\vec{\gamma} \in \Sph^{d-1}$:
\begin{equation}
    \min_{\Phi(\vec{\gamma}, \dummy)} \: \norm{\partial_t^m \Phi(\vec{\gamma}, \dummy)}_{\M(\R)} \quad\subj\quad \Phi(\vec{\gamma}, \vec{\gamma}^\T\vec{x}_n) = z_n(\vec{\gamma}), \: n = 1, \ldots, N,
    \label{eq:1D-spline-radon-domain}
\end{equation}
where
\[
    \int_{\Sph^{d-1}} z_n(\vec{\gamma}) \dd \sigma(\vec{\gamma}) = z_n, \: n = 1, \ldots, N
\]
and noting that by finding a solution for each fixed $\vec{\gamma} \in \Sph^{d-1}$, we can find a $\bar{\Phi}$ that attains a lower bound for \cref{eq:radon-domain-opt}\footnote{Since the integral of a $\min$ is less than or equal to the $\min$ of an integral.}, but this $\bar{\Phi}$ is clearly feasible for \cref{eq:radon-domain-opt} and is hence a solution to \cref{eq:radon-domain-opt}. It then follows that $\bar{\Phi}(\vec{\gamma}, \dummy)$ is a polynomial spline of order $m$. In particular, due to the structure of the interpolation constraints in \cref{eq:1D-spline-radon-domain}, we see that $\bar{\Phi}(\vec{\gamma}, \dummy)$ is interpolates data with sampling locations at $\curly{\vec{\gamma}^\T\vec{x}_n}_{n=1}^N \subset \R$. This viewpoint allows us to understand additional structural information about the sparse (i.e., single-hidden layer neural network) solutions to \cref{eq:inverse-problem-interp-constraints}. In particular, its classically known\footnote{Since we can always explicitly construct spline solutions with the spline knots bounded by the data.} that the univariate spline $\bar{\Phi}(\vec{\gamma}, \dummy)$ has some set of adaptive knot locations  $\curly{t_\ell(\vec{\gamma})}_{\ell=1}^{K_\vec{\gamma}} \subset \R$ with $K_\vec{\gamma} \leq N - m$ and there are no knots outside the sampling locations\footnote{Notice that the number of knots and the knot locations depend on the direction $\vec{\gamma} \in \Sph^{d-1}$.}, i.e.,
\[
    \abs{t_\ell(\vec{\gamma})} \leq \max_{n=1, \ldots, N} \abs{\vec{\gamma}^\T\vec{x}_n}, \quad \ell = 1, \ldots, K_\vec{\gamma}.
\]
It is then clear that for $\bar{\Phi}$ to be a sparse minimizer of \cref{eq:inverse-problem}, it must satisfy \cref{defn:ridge-spline}. This implies that the \emph{biases} in a ridge spline solution to \cref{eq:inverse-problem} exactly correspond to these knot locations. Thus, we can see that for a ridge spline solution to \cref{eq:inverse-problem} as in \cref{eq:ridge-spline} with $K$ neurons, we have the additional information about a bound on the bias terms, which we summarize in the following lemma.
\begin{lemma} \label[Lemma]{lemma:bias-bound}
    In the ideal sampling scenario, the biases in the sparse solution \cref{eq:ridge-spline} of the variational problem in \cref{eq:inverse-problem} satisfy
    \[
        \abs{b_k} \leq \max_{n=1, \ldots, N} \norm{\vec{x}_n}_2,
    \]
    for all $k = 1, \ldots, K$.
\end{lemma}
\begin{proof}
    The proof follows from the discussion above. In particular,
    \[
        \abs{b_k} \leq \sup_{\vec{\gamma} \in \Sph^{d-1}} \max_{\ell=1, \ldots, K_\vec{\gamma}} \abs{t_\ell(\vec{\gamma})} \leq \sup_{\vec{\gamma} \in \Sph^{d-1}} \max_{n=1, \ldots, N} \abs{\vec{\gamma}^\T\vec{x}_n} \leq \max_{n=1, \ldots, N} \norm{\vec{x}_n}_2.
    \]
    for all $k = 1, \ldots, K$.
\end{proof}

 \section{Applications to Neural Networks} \label{sec:nn-training}
In this section we will apply \cref{thm:rep-thm} to neural network training, regularization, and generalization.

\subsection{Finite-dimensional neural network training problems}
In this section we will prove \cref{prop:equiv-opts}.
\begin{lemma} \label[Lemma]{lemma:nn-norm}
Consider the single-hidden layer neural network
\[
    f_\vec{\theta}(\vec{x}) \coloneqq \sum_{k=1}^K v_k\, \rho_m(\vec{w}_k^\T \vec{x} - b_k) + c(\vec{x}),
\]
where $\vec{\theta} = (\vec{w}_1, \ldots, \vec{w}_K, v_1, \ldots, v_K, b_1, \ldots, b_K, c)$ contains the neural network parameters such that $v_k \in \R$, $\vec{w}_k \in \R^d$, and $b_k \in \R$ for $k = 1, \ldots, K$, and where $c$ is a polynomial of degree strictly less than $m$. Also assume without loss of generality that the weight-bias pairs $(\vec{w}_k, b_k)$ are unique\footnote{In the sense that $(\vec{w}_k, b_k) \neq (\vec{w}_n, b_n)$ for $k \neq n$.}. Then,
\[
    \norm{\ROp_m f_\vec{\theta}}_{\M(\cyl)} = \sum_{k=1}^K \abs{v_k} \norm{\vec{w}_k}_2^{m-1}.
\]
\end{lemma}
\begin{proof}
    This proof is a direct calculation. Write
    \begin{align*}
        \ROp_m f_{\vec{\theta}}
        &= \sum_{k=1}^K v_k \ROp_m \rho_m(\vec{w}_k^\T(\dummy) - b_k) \\
        &= \sum_{k=1}^K v_k \norm{\vec{w}_k}_2^{m-1} \ROp_m \rho_m(\tilde{\vec{w}}_k^\T(\dummy) - \tilde{b}_k) \\
        &= \sum_{k=1}^K v_k \norm{\vec{w}_k}_2^{m-1} \sq{\frac{\delta_\cyl(\dummy - (\tilde{\vec{w}}_k, \tilde{b}_k)) + (-1)^m \delta_\cyl(\dummy + (\tilde{\vec{w}}_k, \tilde{b}_k))}{2}},
    \end{align*}
    where the second line follows from the substitution $\tilde{\vec{w}}_k \coloneqq \vec{w}_k / \norm{\vec{w}_k}_2 \in \Sph^{d-1}$ and $\tilde{b}_k \coloneqq b_k / \norm{\vec{w}_k}_2 \in \R$ combined with the homogenity of degree $m - 1$ of $\rho_m$ and the third line follows from \cref{lemma:nn-atoms-sparsified}. Taking the $\M$-norm proves the lemma.
\end{proof}

\subsubsection{Proof of \cref{prop:equiv-opts}}
\begin{proof}
Recasting the problem in \cref{eq:inverse-problem} as \cref{eq:nn-problem} follows from \cref{thm:rep-thm}. Equivalence of the problem in \cref{eq:nn-problem} and \cref{eq:nn-training-with-pathnorm} follows from \cref{lemma:nn-norm}. Thus, we just need to show that the solutions to the problem in \cref{eq:nn-training-with-weight-decay} are also solutions to problem in \cref{eq:nn-training-with-pathnorm}. To see this, let $\vec{\theta}$ be a solution to \cref{eq:nn-training-with-weight-decay} with network weights $\{(v_k,\vec{w}_k)\}_{k=1}^K$. Consider the regularizer from \cref{eq:nn-training-with-weight-decay}:
\[
\frac{1}{2}\sum_{k=1}^K \paren{\abs{v_k}^2 + \norm{\vec{w}_k}_2^{2m-2}}.
\]
Since $\rho_m$ is homogeneous of degree $m-1$, the weights may be rescaled so that $|v_k|=\norm{\vec{w}_k}_2^{m-1}$, $k=1,\dots,K$, without altering the function of the network and its fit to the data.  Note that each term of the regularizer is a sum of squares $\abs{v_k}^2 + \big(\norm{\vec{w}_k}_2^{m-1}\big)^2$, and thus each term is minimized when $|v_k|=\norm{\vec{w}_k}_2^{m-1}$. Thus, at the minimizer we have
\[
\frac{1}{2}\sum_{k=1}^K \paren{\abs{v_k}^2 + \norm{\vec{w}_k}_2^{2m-2}} = \sum_{k=1}^K \abs{v_k} \norm{\vec{w}_k}_2^{m-1},
\]
which is exactly the regularizer of \cref{eq:nn-training-with-pathnorm}.
\end{proof}

\subsection{Generalization bounds for binary classification} \label{sec:generalization}
In this section we will prove \cref{thm:rad}.

\subsubsection{Proof of \cref{thm:rad}}
\begin{proof}
Using the rescaling technique discussed in \cref{rem:rescale},
without loss of generality, we may assume that $\norm{\vec{w}_k}_2 = 1$  (since we can absorb the norm of $\vec{w}_k$ into the magnitude of $v_k$). In this case,
\[
\norm{\ROp_m f_\vec{\theta}}_{\M(\cyl)} = \sum_{k=1}^K |v_k| \leq B.
\]

To begin we bound the \emph{empirical Rademacher complexity} of $\F_\Theta$. The empirical Rademacher complexity, denoted by $\widehat \Rad(\F_\Theta)$, is computed by taking the conditional expectation, conditioning on $\curly{\vec{x}_n}_{n=1}^N$ in place of the total expectation in \cref{eq:rad}. In other words, the only random variables are $\{\sigma_n\}_{n=1}^N$. The Rademacher complexity is then
\[
    \Rad(\F_\Theta)= \E\left[\hat \Rad(\F_\Theta)\right].
\]

We will first consider the empirical Rademacher complexity of a single neuron, i.e., functions of the form $\vec{x}\mapsto \rho_m(\vec{w}^\T\vec{x}-b)$, with $\norm{\vec{w}}_2=1$ and $|b|\leq C/2$.  Write $\E_\vec{\sigma}\sq{\,\dummy\,}$ for $\E\sq{\:\dummy \given \curly{\vec{x}_n}_{n=1}^N}$. The empirical Rademacher complexity of a single neuron is defined to be
\[
\hat{\Rad}\paren{\rho_m(\vec{w}^\T(\dummy)-b)}
\coloneqq 2\, \E_\vec{\sigma}\left[\sup_{\substack{\vec{w}:\norm{\vec{w}}_2 = 1 \\ b:\abs{b} \leq C/2}} \frac{1}{N} \sum_{n=1}^N\sigma_n \rho_m(\vec{w}^\T\vec{x}_n-b) \right].
\]
First notice that when $m$ is odd
\begin{equation}
    \rho_m(\vec{w}^\T\vec{x} - b) = \frac{(\vec{w}^\T\vec{x} - b)^{m - 1} + \abs{\vec{w}^\T\vec{x} - b}^{m - 2}(\vec{w}^\T\vec{x} - b)}{2(m-1)!},
    \label{eq:decomp-odd}
\end{equation}
and when $m$ is even
\begin{equation}
    \rho_m(\vec{w}^\T\vec{x} - b) = \frac{(\vec{w}^\T\vec{x} - b)^{m - 1} + \abs{\vec{w}^\T\vec{x} - b}^{m - 1}}{2(m - 1)!}.
    \label{eq:decomp-even}
\end{equation}
It is well-known that for two function spaces $\F$ and $\mathcal{G}$, the empirical Rademacher complexity satisfies $\hat{\Rad}(\F \oplus \mathcal{G}) = \hat{\Rad}(\F) + \hat{\Rad}(\mathcal{G})$, where $\oplus$ is the direct-sum. With this property, we see from \cref{eq:decomp-odd,eq:decomp-even} that the empirical Rademacher complexity of a single neuron is
\begin{align*}
    &\hat{\Rad}\paren{\rho_m(\vec{w}^\T(\dummy)-b)} \\
    &\qquad = \frac{1}{2(m - 1)!} \paren*[\Bigg]{\underbrace{\hat{\Rad}\paren{(\vec{w}^\T(\dummy)-b)^{m - 1}}}_{(*)} + \underbrace{\left.\begin{cases}
    \hat{\Rad}\paren{\abs{\vec{w}^\T(\dummy)-b}^{m - 2}(\vec{w}^\T(\dummy)-b)}, & \text{$m$ is odd} \\
    \hat{\Rad}\paren{\abs{\vec{w}^\T(\dummy)-b}^{m - 1}}, & \text{$m$ is even}
    \end{cases}\right\}}_{(\mathsection)}}.
\end{align*}
Since the functions in $(*)$ and $(\mathsection)$ are the same up to a sign, it follows that the Rademacher complexities are the same due to the symmetry of the Rademacher random variables. Thus,
\[
    \hat{\Rad}\paren{\rho_m(\vec{w}^\T(\dummy)-b)}
    = \frac{\hat{\Rad}\paren{(\vec{w}^\T(\dummy)-b)^{m-1}}}{(m - 1)!}
    = \frac{2}{N(m-1)!} \E_\vec{\sigma}\left[\sup_{\substack{\vec{w}:\norm{\vec{w}}_2 = 1 \\ b:\abs{b} \leq C/2}} \sum_{n=1}^N\sigma_n (\vec{w}^\T\vec{x}_n-b)^{m-1} \right].
\]
Next, by the binomial theorem
\begin{align*}
    \hat{\Rad}\paren{\rho_m(\vec{w}^\T(\dummy)-b)}&\leq \frac{2}{N(m-1)!}\sum_{k=0}^{m-1} \binom{m - 1}{k} \E_\vec{\sigma}\left[\sup_{\substack{\vec{w}:\norm{\vec{w}}_2 = 1 \\ b:\abs{b} \leq C/2}} \sum_{n=1}^N\sigma_n (\vec{w}^\T\vec{x}_n)^k (-b)^{m-1-k}\right] \\
    &\leq \frac{2}{N(m-1)!}\sum_{k=0}^{m-1} \binom{m - 1}{k} \paren{\frac{C}{2}}^{m - 1 - k} \E_\vec{\sigma}\left[\sup_{\vec{w}:\norm{\vec{w}}_2 = 1} \sum_{n=1}^N\sigma_n (\vec{w}^\T\vec{x}_n)^k \right] \\
    &= \frac{2}{N(m-1)!}\sum_{k=0}^{m-1} \binom{m - 1}{k} \paren{\frac{C}{2}}^{m - 1 - k} \E_\vec{\sigma}\left[\sup_{\vec{w}:\norm{\vec{w}}_2 = 1} \paren{\sum_{n=1}^N\sigma_n \vec{x}_n^{\otimes k}}^\T \vec{w}^{\otimes k} \right] \\
    &\leq \frac{2}{N(m-1)!}\sum_{k=0}^{m-1} \binom{m - 1}{k} \paren{\frac{C}{2}}^{m - 1 - k} \E_\vec{\sigma}\left[ \norm{\sum_{n=1}^N\sigma_n \vec{x}_n^{\otimes k}}_2 \right],
\end{align*}
where $(\dummy)^{\otimes k}$ denotes the $k$th order Kronecker product. By Jensen's inequality we have
\[
    \E_\vec{\sigma} \left[\norm{\sum_{n=1}^N\sigma_n \vec{x}_n^{\otimes k}}_2 \right]
    \leq \E_\vec{\sigma} \left[\norm{\sum_{n=1}^N\sigma_n \vec{x}_n^{\otimes k}}_2^2 \right]^{1/2}
    = \paren{\sum_{n=1}^N \norm{\vec{x}_n^{\otimes k}}_2^2}^{1/2}
    \leq \sqrt{N} \paren{\frac{C}{2}}^{k},
\]
and so
\[
    \hat{\Rad}\paren{\rho_m(\vec{w}^\T(\dummy)-b)} \leq \frac{2C^{m - 1}}{\sqrt{N} (m - 1)!}.
\]
Therefore, the empirical Rademacher complexity of $\F_\Theta$ is bounded as follows
\[
    \hat{\Rad}(\F_\Theta)
    = \sum_{k=1}^K \abs{v_k} \hat{\Rad}\paren{\rho_m(\vec{w}_k^\T(\dummy)-b_k)}+ \hat{\Rad}(c)
    \leq \frac{2BC^{m-1}}{\sqrt{N}(m - 1)!} + \hat{\Rad}(c).
\]
Taking the expectation of both sides proves the theorem.
\end{proof}

  \section{Conclusion} \label{sec:conclusion}
In this paper we have developed a variational framework in which we propose and study a family of continuous-domain linear inverse problems in order to understand what happens on the function space level when training a single-hidden layer neural network on data. We have exploited the connection between ridge functions and the Radon transform to show that training a single-hidden layer neural network on data with an appropriate regularizer results in a function that is optimal with respect to a total variation-like seminorm in the Radon domain. We also show that this seminorm directly controls the generalizability of these neural networks. Our framework encompasses ReLU networks and the appropriate regularizers correspond to the well-known weight decay and path-norm regularizers. Moreover, the variational problems we study are similar to those that are studied in variational spline theory and so we also develop the notion of a ridge spline and make connections between single-hidden layer neural networks and classical polynomial splines. There are a number of followup research questions that may be asked.

\subsection{Computational issues}
Empirical and theoretical work from the machine learning community has shown that simply running (stochastic) gradient descent on a neural network seems to find global minima~\citep{rethink-generalization, regularization-matters, gd-provably, gd-finds-global-min}, though full theoretical justifications of why these algorithms work currently do not exist. Thus, it remains an open question about designing neural network training algorithms that provably find global minimizers. Such algorithms could then be used to find the sparse solutions the continuous-domain inverse problems studied in this paper.

\subsection{Deep networks}
Another important followup question revolves around deep, multilayer networks. Can a variational framework be used to understand what happens when a deep network is trained on data? Answering this question would require posing a continuous-domain inverse problem and deriving a representer theorem showing that deep networks are solutions. We believe answering this question will be challenging, due to the compositions of ridge functions that arise in deep networks.
 
\acks{This work is partially supported by AFOSR/AFRL grant FA9550-18-1-0166, the NSF Research Traineeship Program under grant 1545481, and the NSF Graduate Research Fellowship Program under grant DGE-1747503. The authors thank Greg Ongie for helpful feedback and discussions related to the initial draft of this paper. The authors also thank the anonymous reviewers for their constructive feedback.}

\appendix

\section{Auxiliary Proofs} \label{app:aux-proofs}

\subsection{Proof of \cref{lemma:ideal-sampling-continuous}}
\begin{proof}
    It suffices to prove that for a fixed $\vec{x}_0 \in \R^d$, the operator $f \mapsto \ang{\delta_{\R^d}(\dummy - \vec{x}_0), f}$ is bounded. Let $(\vec{\phi}, \vec{p})$ be a biorthogonal system for the null space $\N_m$ of $\ROp_m$. By \cref{thm:banach-space}, we know every $f \in \F_m$ admits the unique representation $f = \ROp_{m, \vec{\phi}}^{-1} u + q$, where $u = \ROp_m f$. Moreover, from the proof of \cref{thm:banach-space}, we know that $\ROp_{m, \vec{\phi}}^{-1}$ is bounded and hence its kernel $g_{m, \vec{\phi}}(\vec{x}_0, \dummy)$ is bounded for fixed $\vec{x}_0 \in \R^d$. We also have by construction that $g_{m, \vec{\phi}}(\vec{x}_0, \dummy)$ is continuous. Next, suppose $\abs{g_{m, \vec{\phi}}(\vec{x}_0, \dummy)}$ is bounded by $0 < M_{\vec{x}_0} < \infty$. Then, for any $f \in \F_m$
    \begin{align*}
        \abs{\ang{\delta_{\R^d}(\dummy - \vec{x}_0), f}}
        &= \abs{f(\vec{x}_0)} \\
        &= \abs{\ROp_{m, \vec{\phi}}^{-1}\curly{u}(\vec{x}_0) + q(\vec{x}_0)} \\
        &\leq \abs{\ang{u, g_{m, \vec{\phi}}(\vec{x}_0, \dummy})} + \abs{q(\vec{x}_0)} \\
        &\leq M_{\vec{x}_0} \norm{u}_{\M(\cyl)} + \abs{q(\vec{x}_0)} \\
        &= M_{\vec{x}_0} \norm{\ROp_m u}_{\M(\cyl)} + \abs{q(\vec{x}_0)},
    \end{align*}
    where the fourth line follows since we can extend $u \in \M(\cyl)$ to act continuously on $C_\mathrm{b}(\cyl)$, the space of bounded continuous functions on $\cyl$. From \cref{lemma:finite-dim-null-space}, we know $\N_m$ is the space of polynomials of degree strictly less than $m$. Thus, it's clear that point evaluations are continuous on $\N_m$. From the proof of \cref{thm:banach-space}, we know that the norm $q \mapsto \norm{\vec{\phi}(q)}_2$ equips $\N_m$ with a Banach space structure. Therefore, there exists a constant $0< \tilde{M}_{\vec{x}_0} < \infty$ such that $\abs{q(\vec{x}_0)} \leq \tilde{M}_{\vec{x}_0} \norm{\vec{\phi}(q)}_2 = \tilde{M}_{\vec{x}_0} \norm{\vec{\phi}(f)}_2$. Thus, there exists a constant $C_{\vec{x}_0} \coloneqq M_{\vec{x}_0} + \tilde{M}_{\vec{x}_0}$ such that
    \[
        \abs{\ang{\delta_{\R^d}(\dummy - \vec{x}_0), f}}
        \leq C_{\vec{x}_0} \paren{\norm{\ROp_m f}_{\M(\cyl)} + \norm{\vec{\phi}(f)}_2}
        = C_{\vec{x}_0} \norm{f}_{\F_m},
    \]
    where the last inequality is from \cref{item:banach-norm} in \cref{thm:banach-space}.
\end{proof}

\subsection{Proof of \cref{lemma:finite-dim-null-space}}
We can prove that the null space $\N_m$ of $\ROp_m$ is finite-dimensional via the \emph{Fourier slice theorem}~\citep[pg.~4]{integral-geometry-radon-transforms}, which says for any $f \in \Sch'(\R^d)$,
\[
    \hat{\Radon f}(\vec{\gamma}, \omega) = \hat{f}(\omega \vec{\gamma}),
\]
where the Fourier transform on the left-hand side is a one-dimensional Fourier transform with respect to $t \to \omega$ and the Fourier transform on the right-hand side is the multivariate Fourier transform of $f$.
\begin{proof}[Proof of \cref{lemma:finite-dim-null-space}]
    Recall that
    \[
        \ROp_m = c_d\,\partial_t^m \ramp^{d-1} \RadonOp.
    \]
    Next, let $f \in L^{\infty, m - 1}(\R^d)$. Then,
    \[
        \reallywidehat{\ROp_m\curly{f}}(\vec{\gamma}, \omega)
        = c_d \, (\imag\omega)^m \, \imag^d \abs{\omega}^d \, \reallywidehat{\Radon{f}}(\vec{\gamma}, \omega)
        = c_d \, (\imag\omega)^m \, \imag^d \abs{\omega}^d \, \hat{f}(\omega\vec{\gamma}),
    \]
    where the Fourier transform is with respect to $t \to \omega$ and the second equality holds via the Fourier slice theorem. Next, we notice that for $f \in \N_m$ we require that the above display is $0$ for all $(\vec{\gamma}, \omega) \in \cyl$. Since the above display is zero at $\omega = 0$ we see that it must be that $\hat{f}$ is supported only at $\vec{0}$. This says that $f$ must be a polynomial. Finally, since $f \in L^{\infty, m - 1}(\R^d)$, we see that $f$ must be a polynomial of degree strictly less than $m$.
\end{proof}

\subsection{Proof of \cref{lemma:right-inverse}}
\begin{proof}
Using \cref{lemma:nn-atoms-sparsified}, which says that $r_{\vec{z}}$ is a translated Green's function of $\ROp_m$, and that $p_n \in \N_m$, $n = 1, \ldots, N_0$, a direct calculation results in
    \[
      \ROp_m\ROp_{m, \vec{\phi}}^{-1}\psi = \psi
    \]
    for all even (respectively odd) $\psi \in \Sch_0(\cyl)$ when $m$ is even (respectively odd).
    
    Suppose $m$ is even (since the case of $m$ being odd is analogous). To check the boundary conditions we check for all even $\psi \in \Sch_0(\cyl)$ that $\ang{\phi_k, \ROp_{m, \vec{\phi}}^{-1} \psi} = 0$, $k = 1, \ldots, N_0$. Write
    \[
      \ang{\phi_k, \ROp_{m, \vec{\phi}}^{-1} \psi}
      = \ang{\phi_k, \int_{\cyl}
      r_{\vec{z}}(\dummy) \psi(\vec{z}) \dd(\sigma \times \lambda)(\vec{z})} -
      \underbrace{\sum_{n=1}^{N_0} \ang{\phi_k, p_n}\ang{q_n, \psi}}_{=\,
      \ang{q_k, \psi}}.
    \]
    Next,
    \begin{align*}
      \ang{q_k, \psi}
      = \int_{\cyl} q_k(\vec{z})
      \psi(\vec{z}) \dd(\sigma \times \lambda)(\vec{z})
      &= \int_{\cyl} \ang{\phi_k, r_\vec{z}}
      \psi(\vec{z}) \dd(\sigma \times \lambda)(\vec{z}) \\
      &= \ang{\phi_k, \int_{\cyl} r_\vec{z}(\dummy)
      \psi(\vec{z}) \dd(\sigma \times \lambda)(\vec{z})}.
    \end{align*}
    Thus, the previous two displays imply $\ang{\phi_k, \ROp_{m, \vec{\phi}}^{-1} \psi} = 0$, $k = 1, \ldots, N_0$. Uniqueness of $\ROp_{m, \vec{\phi}}^{-1}$ follows from the fact that the biorthogonal system provides a unique representation of elements of $\N_m$.
    
    Again suppose that $m$ is even (since the case of $m$ being odd is analagous). Assume that $\ROp_{m, \vec{\phi}}^{-1}$ is bounded, in other words, that this inverse is \emph{stable}. Since $\Sch_0(\P^d) \subset \Sch(\P^d)$ and $\Sch_0(\P^d)$ and $\Sch(\P^d)$ are both dense in $C_0(\P^d)$ it follows that $\Sch_0(\P^d)$ is dense in $\Sch(\P^d)$. Next, it is well known that the space of Schwartz functions is dense in the space of tempered distributions~\citep{theory-of-distributions}. Thus, since we are assuming $\ROp_{m, \vec{\phi}}^{-1}$ is bounded, we can continuously extend it to act on elements of $\M(\P^d) \subset \Sch'(\P^d)$ with the properties in \cref{eq:stable-right-inverse-props} holding for every $\psi \in \M(\P^d)$. We also remark that we show that $\ROp_{m, \vec{\phi}}^{-1}$ is necessarily bounded when acting on elements of $\M(\P^d)$ in \cref{thm:banach-space}.
\end{proof}

\subsection{Proof of \cref{thm:banach-space}}
\begin{proof}
    Before proving the individual items in the theorem, first recall from the theorem statement the definition
    \[
        \F_{m,\vec{\phi}} = \curly{f \in \F \st \vec{\phi}(f) = \vec{0}}.
    \]
    Clearly this is a vector space. Since $f \mapsto \norm{\ROp_m f}_{\M(\cyl)}$ is a seminorm, it is a a norm except for lacking the property that $\norm{\ROp_m f}_{\M(\cyl)} = 0$ if and only if $f \equiv 0$ (since every $q \in \N_m$ has $\norm{\ROp_m q}_{\M(\cyl)} = 0$). By imposing the boundary conditions $\vec{\phi}(f) = \vec{0}$ in the definition of $\F_{m, \vec{\phi}}$, we enforce that every $f \in \F_{m, \vec{\phi}}$ has no null space component. Thus, $\F_{m, \vec{\phi}}$ is a bona fide Banach space when equipped with the norm $f \mapsto \norm{\ROp_m f}_{\M(\cyl)}$, more specifically this shows that $\F_{m, \vec{\phi}}$ is isometrically isomorphic to $\Me(\cyl)$ (respectively $\Mo(\cyl)$) when $m$ is even (respectively odd)\footnote{Where we are using the fact that the range of $\ROp_m$ is even or odd elements of $\M(\cyl)$ when $m$ is even or odd combined with \cref{eq:odd-even-norm}.}. In particular, this says we have the topological isomorphism $\F_{m,\vec{\phi}} \cong \F_m / \N_m$.
        
    \begin{enumerate}[label=\arabic*.]
        \item The above discussion along with \cref{eq:odd-even-norm} has shown that $\ROp_m$ is a bijective isometry from $\F_{m, \vec{\phi}}$ to $\Me(\cyl)$ or $\Mo(\cyl)$ when $m$ is even or odd. Since $\ROp_m$ is a bijective isometry, it is bounded, allowing us to invoke the bounded inverse theorem~\cite[Chapter~5]{folland}. This says there exists a bounded inverse $\ROp_m^{-1}$ (in this case, an isometry) of $\ROp_m$ mapping $\Me(\cyl)$ or $\Mo(\cyl)$ to $\F_{m, \vec{\phi}}$ when $m$ is even or odd. This inverse is necessarily the \emph{unique operator} $\ROp_{m, \vec{\phi}}^{-1}$ constructed in \cref{lemma:right-inverse} as it imposes the boundary conditions in \cref{eq:stable-right-inverse-props}. This result indirectly shows that the operator $\ROp_{m, \vec{\phi}}^{-1}$ when acting on $\Me(\cyl)$ or $\Mo(\cyl)$ is bounded when $m$ is even or odd.
        
        \item  Given the biorthogonal system $(\vec{\phi}, \vec{p})$ of $\N_m$, consider the projection operator
        \[
            \proj_{\N_m}: f \mapsto \sum_{n=1}^{N_0} \ang{\phi_n, f} p_n.
        \]
        Then, for every $f \in \F_m$ we can write $f =
        \tilde{f} + q$, where $q \coloneqq \proj_{\N_m}$ and so $\tilde{f} = f -
        q$. By this construction, $\vec{\phi}(\tilde{f}) = 0$, so $\tilde{f} \in
        \F_{m, \vec{\phi}}$. Clearly, $\tilde{f} =
        \ROp_{m, \vec{\phi}}^{-1} u$, where $u \coloneqq \ROp_m \tilde{f} = \ROp_m f$. Indeed, this is true since
        $\ROp_{m, \vec{\phi}}^{-1}$ is a bona fide inverse when acting on 
        $\Me(\cyl)$ or $\Mo(\cyl)$ when $m$ is even or odd. Thus, \cref{eq:unique-representation} holds.
        Since $\F_{m, \vec{\phi}} \cong
        \F_m / \N_m$, we have
        $\F_{m, \vec{\phi}} \cap \N_m = \curly{0}$. Hence,
        we have the direct-sum decomposition $\F_m =
        \F_{m, \vec{\phi}} \oplus \N_m$.
        
        \item Since every $q \in \N_m$ has the unique representation
        \[
            q = \sum_{n = 1}^{N_0} \ang{\phi_n, q} p_n,
        \]
        we can always identify $q$ by its expansion coefficients $\vec{\phi}(q) \in \R^{N_0}$. Thus, the norm $q \mapsto \norm{\vec{\phi}(q)}_2$ provides $\N_m$ with a Banach space structure. Finally, since both $\F_{m, \vec{\phi}}$ and $\N_m$ can be endowed with norms to provide a Banach space structure, we can use the direct-sum decomposition $\F_m =
        \F_{m, \vec{\phi}} \oplus \N_m$ to equip $\F_m$ with the composite norm
        \[
            \norm{f}_{\F_m} \coloneqq \norm{\ROp_m f}_{\M(\cyl)} + \norm{\vec{\phi}(f)}_2
        \]
        to provide $\F_m$ a Banach space structure.
    \end{enumerate}
\end{proof}

\bibliography{ref}

\end{document}